\documentclass[10pt,journal,cspaper,twocolumn,compsoc]{IEEEtran}
\usepackage[pagebackref=false,colorlinks,citecolor=blue,linkcolor=blue,bookmarks=false]{hyperref}
\usepackage{etex}
\usepackage{subfig}
\usepackage{graphicx}
\usepackage{ctable} 
\usepackage{multirow} 
\usepackage{amsmath}
\usepackage{amssymb}
\usepackage{amsthm} 
\usepackage[roman]{parnotes} 
\usepackage{paralist} 
\usepackage{grffile}
\usepackage{adjustbox}
\usepackage{float}
\usepackage{overpic}
\newsavebox{\measurebox}
\usepackage[font={small}]{caption}
\usepackage{tikz}
\usepackage{pgfplots}
\usetikzlibrary{plotmarks,patterns}
\pgfplotsset{compat=newest}

\definecolor{mygreen}{RGB}{69,182,73}
\definecolor{color_SH_E}{rgb}{0,0,0}
\definecolor{color_SH_D}{rgb}{1,0,0}
\definecolor{color_LSC}{rgb}{0.6552,0.3793,0.2414}
\definecolor{color_SEEDS}{rgb}{1,0.8276,0}
\definecolor{color_ERS}{rgb}{0,0.5172,0.9655}
\definecolor{color_SLIC}{rgb}{1,0.1034,0.7241}
\definecolor{color_FH}{RGB}{69,182,73}
\definecolor{light_red}{RGB}{253,232,233}
\definecolor{light_green}{RGB}{233,248,241}

\def\etal{\textit{et al.} }
\theoremstyle{plain} \newtheorem{theorem}{Theorem}
\theoremstyle{plain} \newtheorem{lemma}{Lemma}

\begin{document}
\title{Superpixel Hierarchy}

\author{
        Xing~Wei,
        Qingxiong~Yang,~\IEEEmembership{Member,~IEEE,}
        Yihong~Gong,~\IEEEmembership{Member,~IEEE,}
        Ming-Hsuan~Yang,~\IEEEmembership{Senior Member,~IEEE,}
        Narendra~Ahuja,~\IEEEmembership{Fellow,~IEEE,}\\
{\small\url{http://www.cs.cityu.edu.hk/~qiyang/publications/SH/}}
}

\IEEEtitleabstractindextext{%
\begin{abstract}



Superpixel segmentation is becoming ubiquitous in computer vision.
In practice, an object can either be represented by a number of segments in finer levels of detail or included in a surrounding region at coarser levels of detail, and thus a superpixel segmentation hierarchy is useful for applications that require different levels of image segmentation detail depending on the particular image objects segmented.
Unfortunately, there is no method that can generate all scales of superpixels accurately in real-time.
As a result, a simple yet effective algorithm named Super Hierarchy (SH) is proposed in this paper.
It is as accurate as the state-of-the-art but 1-2 orders of magnitude faster\footnote{Hundreds speedup can be obtained for multi-scale superpixels.}.
The proposed method can be directly integrated with recent efficient edge detectors like the structured forest edges \cite{DollarICCV13edges} to significantly outperforms the state-of-the-art in terms of segmentation accuracy. Quantitative and qualitative evaluation on a number of computer vision applications was conducted, demonstrating that the proposed method is the top performer.

\end{abstract}

}

\maketitle

\IEEEdisplaynontitleabstractindextext

%
\IEEEpeerreviewmaketitle

\ifCLASSOPTIONcompsoc
\IEEEraisesectionheading{\section{Introduction}\label{sec:introduction}}
\else
\section{Introduction}
\label{sec:introduction}
\fi

%
%
%
%

\begin{table*}[!t] \footnotesize
\centering
\def\arraystretch{1.5}
\renewcommand{\tabcolsep}{2.5 pt}

\scalebox{0.9}{%
\begin{tabular}{l|l|c|cccccc}
\specialrule{.1em}{.1em}{.1em}
\hline
    & \multicolumn{2}{c|}{} & FH \cite{felzenszwalb2004efficient} & SLIC \cite{achanta2012slic} & ERS \cite{liu2011entropy} & SEEDS \cite{van2012seeds} & LSC \cite{LiC15} & Our SH\\
 \hline
 \multirow{4}{3cm}{Property 1: \\Segmentation accuracy} & \multicolumn{2}{l|}{Achievable segmentation accuracy (average on BSDS500~\cite{pami-11-malik})}    & 93.5\% & 94.1\% & \color{blue}{94.9\%} & 94.7\% & \textcolor[rgb]{0.2,0.6,0.2}{95.0}\% & \color{red}{95.1\%} \\

                                                        & \multicolumn{2}{l|}{Under-segmentation error (average on BSDS500)}            & 12.6\% & 11.5\% & \color{blue}{10.1}\% & 10.4\% & \textcolor[rgb]{0.2,0.6,0.2}{10.0\%} & \color{red}{9.7\%} \\
                                                        &  \multicolumn{2}{l|}{Boundary recall (average on BSDS500)} & \color{blue}{78.0\%} & 67.2\% & 75.5\% & 72.9\% & \textcolor[rgb]{0.2,0.6,0.2}{78.7\%} & \color{red}{80.8\%} \\
                                                        & \multicolumn{2}{l|}{Semantic segmentation accuracy
                                                        (using \cite{gould2008multi} on MSRC \cite{shotton2006textonboost})} & 63.8\% & 65.4\% & \color{blue}{65.5\%} & \textcolor[rgb]{0.2,0.6,0.2}{65.8\%} & 65.3\% & \color{red}{66.6\%} \\
 \hline
 \multirow{2}{3cm}{Property 2: \\Segmentation speed}    & \multicolumn{2}{l|}{$321 \times 481$ image (average on BSDS500)}   & 328 ms \parnote{Reported time includes parameter search} & \color{blue}{108 ms} & 689 ms & \textcolor[rgb]{0.2,0.6,0.2}{52 ms} & 302 ms & \color{red}{31 ms} \\
                                                                           & \multicolumn{2}{l|}{Complexity of obtaining $n$ scales of superpixels} & $O(n)$  & $O(n)$  & $O(n)$  & $O(n)$ & $O(n)$ & \color{red}{$O(1)$} \\
 \hline
 \multirow{2}{4cm}{Property 3: \\Hierarchical segmentation} & Saliency detection & run-time of superpixels \parnote{Using 5 scales of superpixels} & 1810 ms\textsuperscript{i}  & \color{blue}{554 ms} & 4160 ms & \textcolor[rgb]{0.2,0.6,0.2}{282 ms} & 1720 ms & \color{red}{35 ms} \\
                                                              & (using \cite{qin2015saliency} on PASCAL-S \parnote{Average size: $500 \times 356$} \cite{li2014secrets}) & mean absolute error\textsuperscript{ii} & \color{blue}{0.187} & 0.190 & 0.189 & \textcolor[rgb]{0.2,0.6,0.2}{0.186} & 0.190 & \color{red}{0.181} \\
 \hline
 Property 4: Maintain topology                           & \multicolumn{2}{l|}{Stereo matching error (using \cite{yang2012non} on Middlebury \cite{middlebury})} & \textcolor[rgb]{0.2,0.6,0.2}{8.10\%} & - & \color{blue}{8.69\%} & - & - & \color{red}{7.63\%} \\
 \hline
 \specialrule{.1em}{.1em}{.1em}
\end{tabular}}
\parnotes

\caption{\textbf{Super Hierarchy compared to state-of-the-art superpixel algorithms.}
\emph{Property 1}: Segmentation accuracy is measured according to three standard metrics: achievable segmentation accuracy, under-segmentation error and boundary recall on the BSDS500~\cite{pami-11-malik} and the segmentation accuracy on the MSRC-21~\cite{shotton2006textonboost} using the method proposed in \cite{gould2008multi}.
\emph{Property 2}: We report the average runtime required to segment images on an Intel i7 3.4 GHz CPU (using single core without SIMD instructions).
The theoretical complexity of obtaining multi-scale superpixels of each method is also provided.
\emph{Property 3}: The advantage of multi-scale segmentation is evidenced by saliency detection application \cite{qin2015saliency}.
\emph{Property 4}: We demonstrate the usefulness of tree structures provided by FH, ERS, and SH with the non-local cost aggregation algorithm \cite{yang2012non} for stereo matching and numerically evaluating them on the Middlebury benchmark~\cite{middlebury}. The top three algorithms are highlighted in \textcolor{red}{red}, \textcolor{mygreen}{green} and \textcolor{blue}{blue} respectively.
}
\label{tab:property}
\end{table*}

\IEEEPARstart{S}{uperpixel} segmentation is becoming ubiquitous in computer vision.
Superpixels are perception meaningful groupings of pixels and serve as primitives for further computation. Superpixels are key building blocks of many algorithms as they significantly reduce the number of image primitives compared to pixels.
This paper aims at developing computationally efficient approaches to superpixel segmentation that can be used in a wide range of computer vision tasks.
In order to achieve such versatile utility, a superpixel method should have the following properties:
\begin{asparaitem}

\item \label{prop:accurate} \textbf{Boundary adherent}: is a basic requirement of superpixels, each superpixel should only overlap with one object.
\item \label{prop:efficient} \textbf{Computational efficiency}: is crucial for superpixel segmentation as it is typically served as a pre-processing step.
%
%
The computational complexity should be independent of the number of superpixels and linear/sublinear in the image size.
\item \label{prop:hierarchy} \textbf{Hierarchal segmentation}: is considered to be close to the human visual system.
    Many algorithms can benefit from multi-resolution representations of images and
    superpixels with a natural hierarchy can be applied to more vision tasks.
\item \label{prop:topology} \textbf{Topology preserving}: simplifies the usage of the extracted superpixels.
To be a good substitute of the image pixels, superpixels should conform to a simple topology otherwise the neighborhood information cannot be maintained.
\end{asparaitem}

While the past few years have seen considerable progress in superpixel segmentation \cite{felzenszwalb2004efficient,moore2008superpixel,achanta2012slic,liu2011entropy,van2012seeds,LiC15}, the state-of-the-art methods possess only one to two of these properties which limit their utility for many vision tasks.
For instance, Liu \etal propose a graph-based method \cite{liu2011entropy} that has good segmentation accuracy.
However, it is computationally prohibitive for real-time applications.
The SEEDS method \cite{van2012seeds} achieves a compromise between accuracy and efficiency but its run-time depends on the number of superpixels.
One class of approaches \cite{moore2008superpixel,moore2010lattice} generate superpixels that conform to a grid topology which can be used by many vision algorithms conveniently.
However, the computational complexity of both is high and the segmentation accuracy is inferior to the state-of-the-art~\cite{LiC15}.
In additional to grid topology, some algorithms \cite{todorovic2008unsupervised,yang2012non,mei2013segment} use a tree structure of regions to represent an image.
Felzenszwalb and Huttenlocher \cite{felzenszwalb2004efficient} propose a method that can provide such a structure by adding edges between segments. This structure was adopted in \cite{mei2013segment} to replace the minimum spanning tree used in Yang's non-local stereo matching algorithm\cite{yang2012non}.
%
Nevertheless, its under-segmentation error is high as shown by the recent superpixel benchmark evaluation \cite{neubert2012superpixel}.
Furthermore,  many vision algorithms \cite{kohli2009robust,yan2013hierarchical,jiang2013salient}
benefit from hierarchical or multi-scale segmentation but most superpixel methods do not consider such hierarchical structures.
As a result, superpixel algorithms must be performed several times to generate superpixels at different scales, which increase the computational cost.

\setlength{\fboxrule}{1.0pt}
\setlength{\fboxsep}{0in}
\begin{figure}[!t]
\centering
{\includegraphics[width=0.98\linewidth]{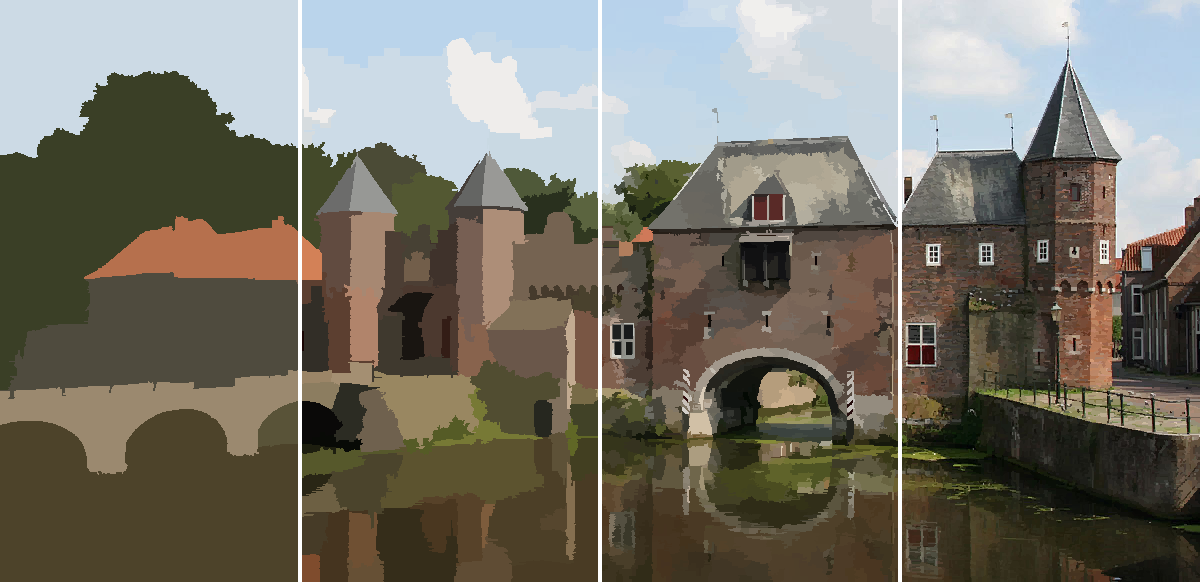}} \vspace{2mm}
\caption{\textbf{Super Hierarchy}: segmentation with 16, 256 ,4096, and 65536 superpixels. Original image size: $1200 \times 582$.}
\label{fig:SH:introduction}
\end{figure}

In contrast, this paper describes a method~\footnote{Source code is available at ~\url{http://www.cs.cityu.edu.hk/~qiyang/publications/SH/}} that enjoys all these properties as summarized in Table \ref{tab:property}.
Our method efficiently builds a superpixel hierarchy that can generate any number of superpixels (between one and the number of pixels) on the fly (as shown in Figure \ref{fig:SH:introduction}).
These superpixels are organized by a tree structure.
Extensive experiments (Section~\ref{sec:experiment}, \ref{sec:application}) demonstrate that our algorithm outperforms state-of-the-art both
in terms of accuracy and efficiency.
%

\section{Related Work}
\label{sec:related}

There is a large literature on image segmentation.
In this section, we briefly discuss the most relevant methods to this work:
hierarchical image segmentation and efficient superpixel extraction.

\subsection{Hierarchical Image Segmentation}

Hierarchical image segmentation methods generate a set of segments with
different details in which the ones on the coarser levels are
composed of regions on the finer levels.


The segmentation algorithm \cite{akbas2010ramp} partitions a given
image into homogeneous regions of a priori unknown shape, size and
degree of photometric homogeneity and organized in a hierarchy
\cite{ahuja1996transform}.
Nodes on the upper levels correspond to large segments while their
children nodes capture finer details.
A connected segmentation tree \cite{ahuja2008connected} includes
additional edges to  sibling nodes by introducing neighboring Voronoi
regions.
The segment tree structure has been applied to image classification,
semantic image segmentation and object detection \cite{akbas2014low}.

The method \cite{pami-11-malik,APBMM2014} transforms the output of any contour
detector into a hierarchical region tree.
To this end, it uses the oriented watershed transform to
construct a set of initial regions
followed by an agglomerative clustering procedure to construct a hierarchical representation.
This hierarchical segmentation method has been successfully used in numerous
recognition and detection problems \cite{gu2009recognition,girshick2014rich}.

%
Although image segmentation plays an important role in computer vision, accurate segmentation methods are often time-consuming which extremely limit their application domains.
On the other hand, more algorithms have resorted to superpixels
which can be generated efficiently with low under-segmentation error.

\subsection{Efficient Superpixel Extraction}

%
We review the state-of-the-art superpixel algorithms
that are either based on image partition or region merging.

{\flushleft \bf{Image Partition.}}
These algorithms start from an initial rough partition of an image, typically with a regular grid and then
refine the segments iteratively, e.g., SLIC \cite{achanta2012slic} and SEEDS \cite{van2012seeds}.
SLIC is an adaptation of the $k$-means clustering for superpixel generation.
It limits the search space of each cluster center and results in a
significant speed-up  over the conventional $k$-means clustering.
SEEDS builds an objective function that can be maximized by
a simple hill-climbing optimization process efficiently.
By avoiding computing distance from centers, it directly exchanges
pixels between superpixels by moving the boundaries.
These methods often have the compactness constraint that favors
equal-sized or regular-shaped segments.
However, the generated superpixels do not adhere well to image boundaries,
especially for fine-structured objects as they are not well modeled
by regular sampling, e.g. Figure~\ref{fig:sh_slic}.

{\flushleft \bf{Region Merging.}}
These algorithms operate on growing regions
into segments.
A representative algorithm is introduced by Felzenszwalb and
Huttenlocher (FH) \cite{felzenszwalb2004efficient}.
FH is a graph-based method in which pixels are vertices and edge
weights measure the dissimilarity between vertices.
Similar to other region merging methods \cite{salembier2000binary,calderero2010region},
it uses the Kruskal's algorithm \cite{west2001introduction} to build a minimum spanning forest
in which each tree is a segment.
Each vertex is initially placed in its own component, and the FH
method merges regions by a criterion that the resulting
segmentation is neither too coarse nor too fine.

Our approach is mostly related to FH which also works in a
growing manner.
FH adaptively adjusts segmentation criterion based on
the degree of variability in neighboring regions of the image,
such that it obeys certain global properties even though makes greedy decisions.
However, the under-segmentation error is high
as shown by the recent studies
\cite{achanta2012slic,van2012seeds}.
%
%
Our proposed algorithm dynamically adjusts the weights of the graph by
aggregating the attributes of clusters during segmentation.
We show that this feature aggregation scheme outperforms the state-of-the-art
methods in all evaluation metrics.
%

\section{Superpixel Hierarchy}
\label{sec:method}

Let $\mathcal{G}=(\mathcal{V},\mathcal{E})$ denote an undirected graph
consisted of vertices $v \in \mathcal{V}$ and edges $e \in \mathcal{E}
\subseteq \mathcal{V} \times \mathcal{V}$ with cardinalities $ n =
\left| \mathcal{V} \right|$ and $m = \left| \mathcal{E} \right|$.
Each pixel is associated with a vertex and locally connected to
its 4 neighbors.
Each edge $e_{ij} = ({v_i},{v_j})$ is assigned a weight (typically
non-negative real number) that measures the dissimilarity between the
two vertices.
%
%
In the superpixel segmentation task, let $k$ denote the number of superpixels to be extracted,
a segmentation $\mathcal{S}$ of a graph $\mathcal{G}$ is a partition of $\mathcal{V}$ into $k$ disjoint components and each component $\mathcal{C} \in \mathcal{S}$ corresponds to a connected subgraph $\mathcal{G}' = (\mathcal{V}',\mathcal{E}')$, where $\mathcal{V}' \subseteq \mathcal{V}$ and $\mathcal{E}' \subseteq \mathcal{E}$.

Our algorithm belongs to a class of region merging methods \cite{felzenszwalb2004efficient,calderero2010region}. Different from traditional methods that use the Kruskal's algorithm \cite{west2001introduction} as merging order, we grow regions in Bor\r{u}vka's fashion \cite{west2001introduction}.
%
%
The advantages are three-fold:
\begin{asparaitem}
\item The Bor\r{u}vka's algorithm has a linear time solution \cite{mares2004} and is parallelizable.
\item Multi-scale information can be incorporated into one unified framework, i.e.,
after each iteration, the graph weights are updated according to newly formed clusters.
\item A hierarchy is built during merging from which any amount of superpixels can be generated on the fly.
\end{asparaitem}

We first review the Bor\r{u}vka's algorithm, and then address the efficiency and accuracy issues by edge contraction and feature aggregation, respectively.

\begin{figure}[!t]
\def\swidth{0.26\linewidth}
\renewcommand{\tabcolsep}{0.0 pt}
\centering
\scalebox{0.76}{
\hspace{-1.5mm}
\begin{tabular}{ccccc}
\includegraphics[width=\swidth]{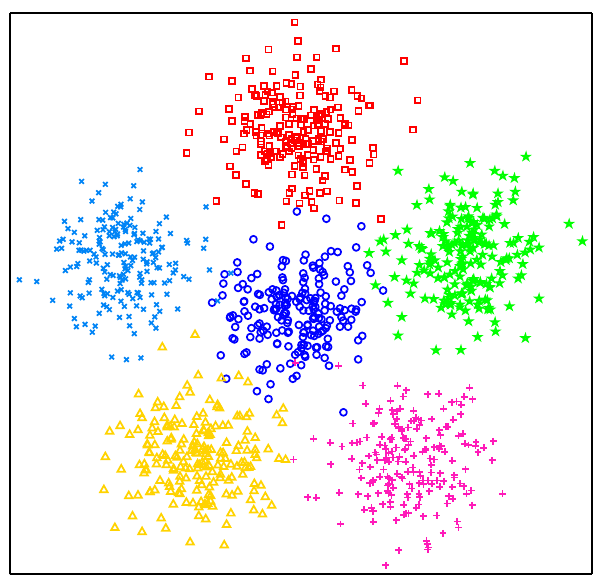}        &
\includegraphics[width=\swidth]{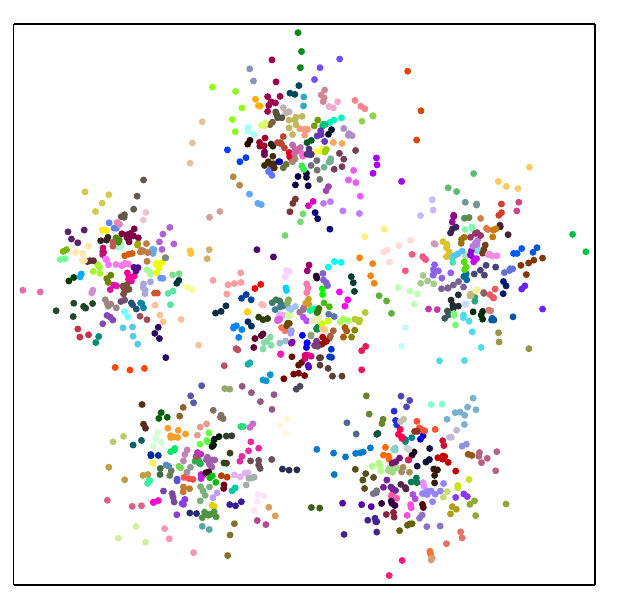}   &
\includegraphics[width=\swidth]{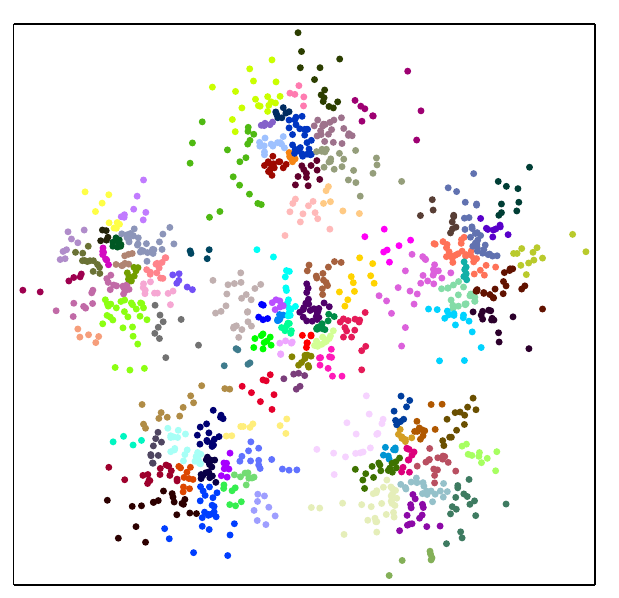}   &
\includegraphics[width=\swidth]{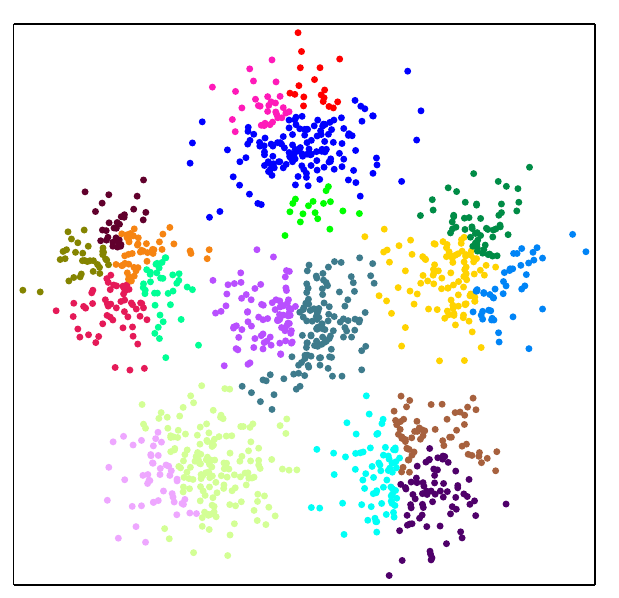}   &
\includegraphics[width=\swidth]{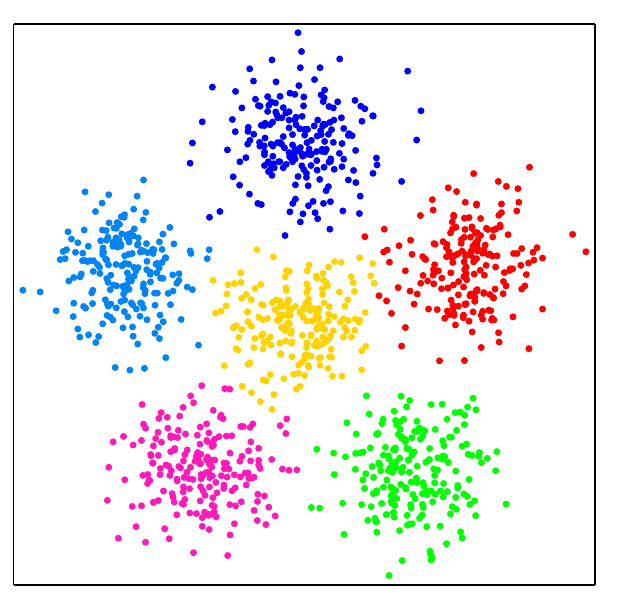}   \\

\multirow{1}{\swidth}{\centering\scriptsize{(a) input:   1200 nodes and 3041 edges}}   &
\multirow{1}{\swidth}{\centering\scriptsize{(b) 1st iter: 373 nodes and 659 edges}} &
\multirow{1}{\swidth}{\centering\scriptsize{(c) 2nd iter:  99 nodes and 189 edges}} &
\multirow{1}{\swidth}{\centering\scriptsize{(d) 3rd iter:  19 nodes and  29 edges}} &
\multirow{1}{\swidth}{\centering\scriptsize{(e) 4th iter:   6 nodes and   8 edges}} \vspace{4mm}\\

\includegraphics[width=\swidth]{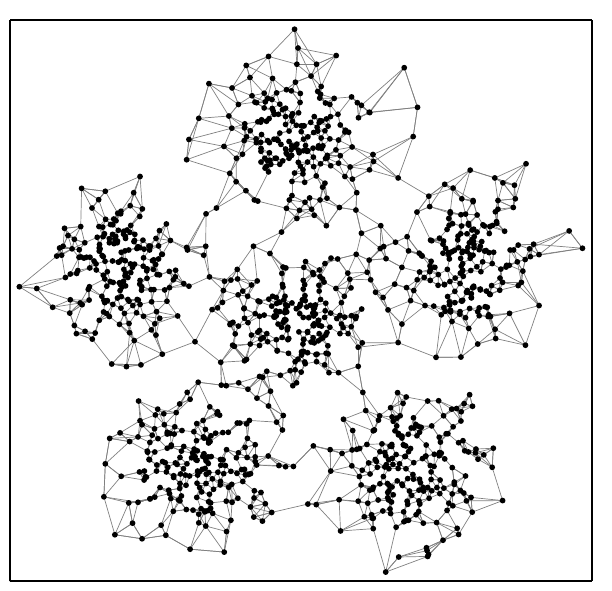}     &
\includegraphics[width=\swidth]{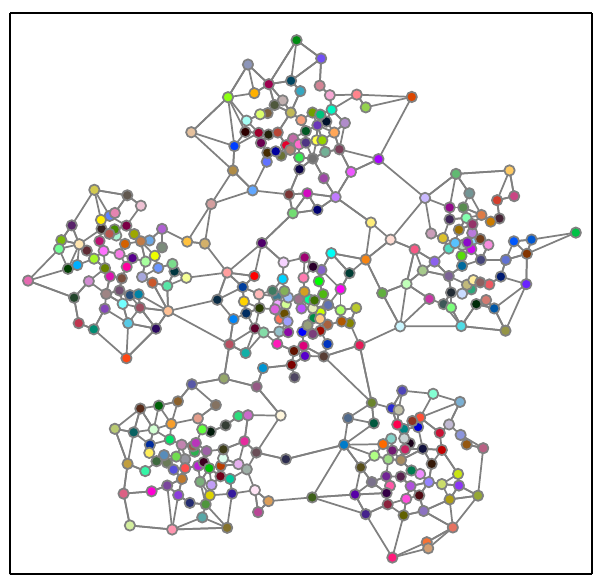}   &
\includegraphics[width=\swidth]{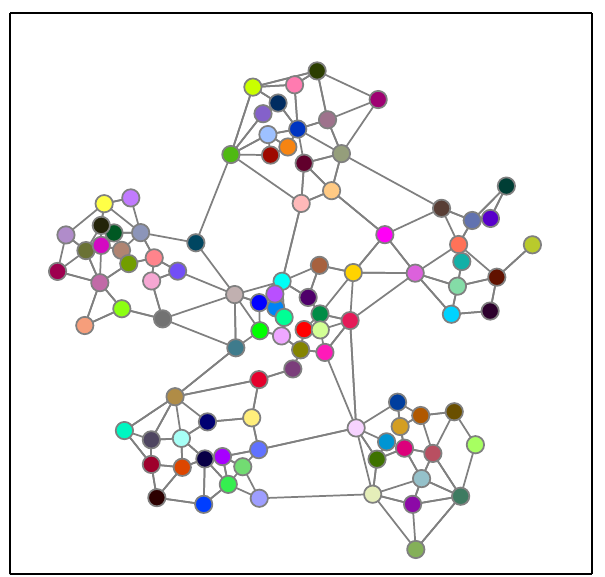}   &
\includegraphics[width=\swidth]{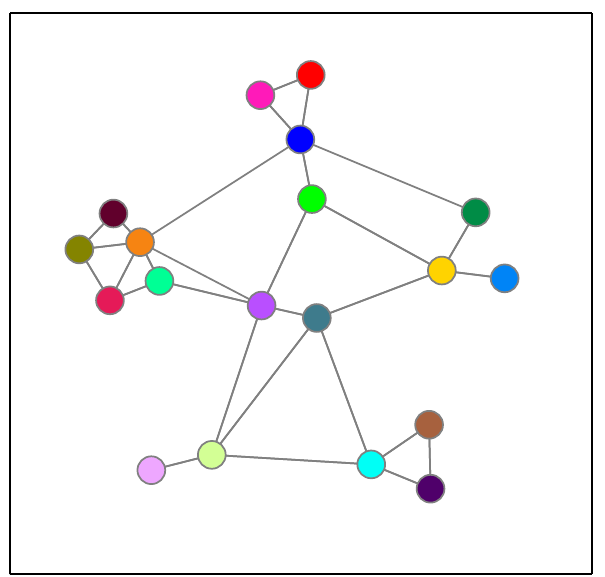}   &
\includegraphics[width=\swidth]{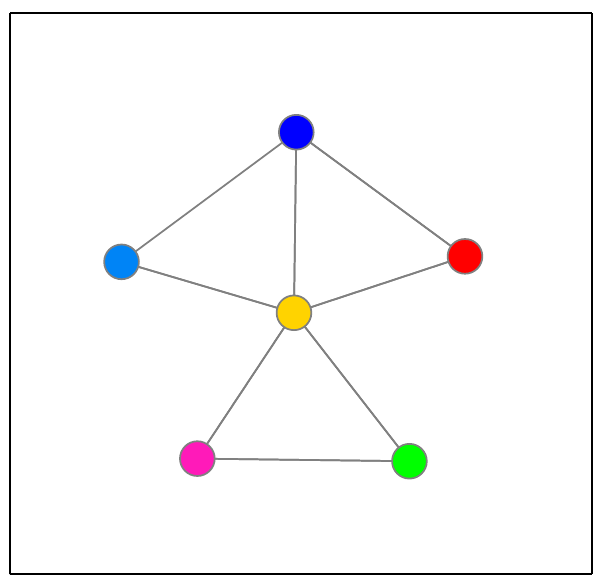} \vspace{1mm}\\

\end{tabular}}
\caption{
\textbf{Towards concurrently generating all scales of superpixels}.
(a) A data set consisting of six Gaussian clouds and its 4 nearest neighbor graph. (b)-(e) First 4 iterations of proposed SH algorithm. Concurrently computing all scales of superpixels is achieved by region merging. Unlike other region merging methods \cite{salembier2000binary,felzenszwalb2004efficient,calderero2010region} that use the Kruskal's algorithm \cite{west2001introduction}, we adopt
the Bor\r{u}vka's algorithm \cite{west2001introduction} to grow a spanning tree.
The advantages are three-fold. 1)
The Bor\r{u}vka's algorithm  has linear time solution~\cite{mares2004} and is parallelizable.
As shown in (b)-(e), the numbers of nodes and edges are decreasing geometrically after each iteration which enables linear time complexity.
2) Multi-scale information can be incorporated into a unified framework: after each iteration, the graph weights are updated according to newly formed clusters. 3) A natural hierarchy is built during merging from which any amount of superpixels can be generated on the fly. }
\end{figure}

\subsection{Superpixels via the Bor\r{u}vka's Algorithm}
\label{sec:boruvka}

The Bor\r{u}vka's algorithm computes a \emph{Minimum Spanning Tree} (MST) in a bottom-up manner.
Consider a graph as a forest with $n$ trees, namely one vertex itself is a tree.
For each tree, we find its nearest neighbor which is connected by the
lightest edge, and join them together.

Formally, let $\mathcal{C}_2$ denote the nearest neighbor of
$\mathcal{C}_1$ ($\mathcal{C}_1$ may not be the nearest neighbor of
$\mathcal{C}_2$).
 We define the distance between two trees as
\begin{equation}
D(\mathcal{C}_1,\mathcal{C}_2) = \mathop {\min }\limits_{{v_i} \in
  \mathcal{C}_1,{v_j} \in \mathcal{C}_2,({v_i},{v_j}) \in \mathcal{E}
} w(({v_i},{v_j})).
\label{dist1}
\end{equation}

%
The Bor\r{u}vka's algorithm repeats merging trees in this manner until
only one tree is left.
The major difference between Bor\r{u}vka's algorithm and Kruskal's algorithm is that
the former searches edges locally and simultaneously while the later sorts the edges globally and conducts sequentially.
This means that the Bor\r{u}vka's algorithm can be processed in parallel.
In addition, the Bor\r{u}vka's algorithm takes the prior that clusters are uniformly distributed
which overcome the drawback of Kruskal's algorithm that tends to produce heavily unbalanced clusters~\cite{west2001introduction}.
In the proposed superpixel hierarchy method,
we use the Bor\r{u}vka's algorithm to build a MST.
At the same time, the order that each edge is added to the MST is
recorded.
Once an edge is added to the MST, the number of trees in the forest is
reduced by one.
Suppose that $k$ superpixels need to be extracted, we connect vertices by the
first $n-k$ edges, this results in $k$ connected components which are
the superpixels exactly.
%

\subsection{Linear Time Algorithm via Edge Contraction}
\label{sec:complexity}

In this section, we re-formulate Bor\r{u}vka's algorithm with regard to edge contraction.
Instead of maintaining a forest of trees, we can keep each tree contracted to a single
vertex.
This reduces the number of vertices and edges substantially,
thereby speeding up the computation.

An edge contraction is illustrated in Figure \ref{fig:con}. Figure
\ref{fig:con}(a) shows a graph with a number in each vertex
representing its attributes (e.g., intensity of a pixel).
Edge weights are computed
by the absolute distance of their two ends.
An edge contraction is performed between vertex $4$ and $2$.
After contracting the edge,
as shown in Figure \ref{fig:con}(b), vertex $4$ and $2$ become a
supervertex, resulting in a self-loop and two parallel edges.
A flattening operation is followed in Figure
\ref{fig:con}(c) by removing the self-loop and replacing parallel edges
by the lightest one.

We explain the details of our implementation along with complexity analysis.
We denote the graph at the beginning of the $i$-th iteration by $\mathcal{G}_i$ and the number of vertices and edges of this graph by $n_i$ and $m_i$, respectively.

\begin{lemma}[] The SH algorithm stops in $O(\log n)$ iterations.
\label{lemma:1}
\end{lemma}

\begin{proof}
Each tree gets merged with at least one of its neighbors, and the
number of trees in $\mathcal{G}_i$ decreases by at least a factor of
two.
Thus, the SH algorithm stops in $O(\log n)$ iterations.
\end{proof}

\begin{lemma}[] Each iteration of SH algorithm runs in $O(m_i)$
  time.
\end{lemma}

\begin{proof}
First, the nearest neighbor search for each tree loops through all edges
and determines whether one edge is the lightest one for the trees on either
endpoint, which takes $O(m_i)$ time.
Next, the histogram sorting \cite{cormen2001introduction} of chosen edges takes $O(n_i)$ time.
%
In addition, tree growing uses an auxiliary graph whose vertices are the labels of the original trees and
edges correspond to the chosen lightest edges. The auxiliary graph has $n_i$ vertices and edges.
We find the connected components of this graph using depth-first search, which takes $O(n_i)$ time~\cite{cormen2001introduction}.

%
%
Edge contraction is performed by histogram sorting the edges lexicographically
and then removing the loops and parallel edges, which takes $O({m_i})$ time.
Thus, each iteration of the SH algorithm takes $O({m_i} + {n_i}) =
O({m_i})$ time.

\end{proof}

\begin{theorem}[] The SH algorithm runs in $O(n)$ on
planar graphs.
\end{theorem}

\begin{proof}
When the input is a planar graph, every $\mathcal{G}_i$ is
planar because the class of planar graphs is closed under edge
contraction~\cite{west2001introduction}.
Moreover, $\mathcal{G}_i$ is also simple (graph loops and parallel edges have already been removed)
such that we can use Euler's formula on the number of edges of planar simple graphs to obtain ${m_i} \le 3{n_i}$.
From Lemma \ref{lemma:1}, we know that ${n_i} \le  n/{2^i}$, and therefore
the total time complexity of the SH algorithm is $O(\sum\nolimits_i {{m_i}} ) = O(\sum\nolimits_i {{n \mathord{\left/
 {\vphantom {n {{2^i}}}} \right.
 \kern-\nulldelimiterspace} {{2^i}}}} ) = O(n).$
\end{proof}

\subsection{Improve Robustness via Feature Aggregation}

\begin{figure}[!t]\footnotesize
\centering
\def\arraystretch{0.5}
\renewcommand{\tabcolsep}{4.0 pt}
\def\swidth{0.28\linewidth}
\begin{tabular}{cc}
\hspace{-2mm}
\begin{adjustbox}{valign=c}
\begin{tabular}{cc}
\includegraphics[width=\swidth]{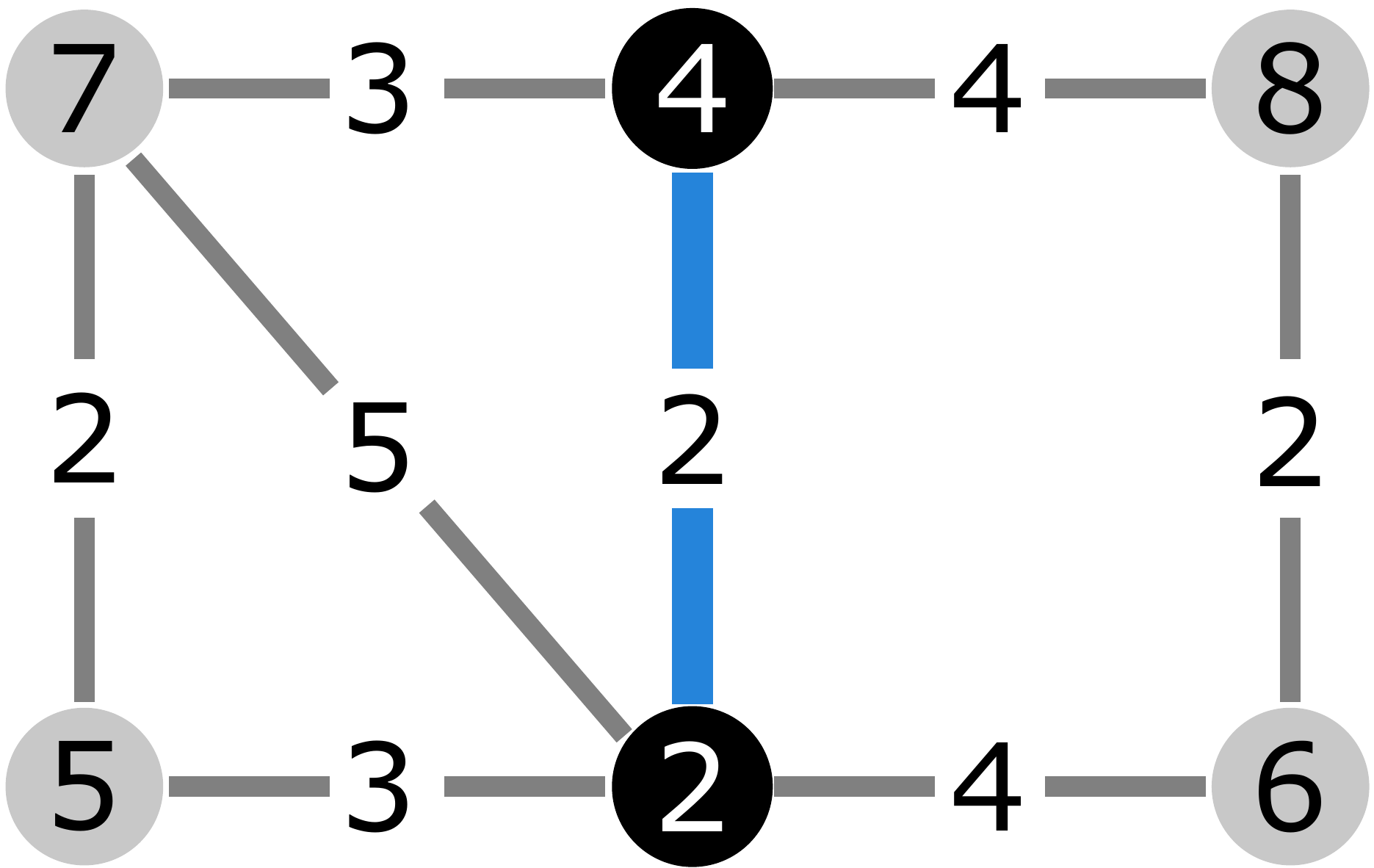} &
\includegraphics[width=\swidth]{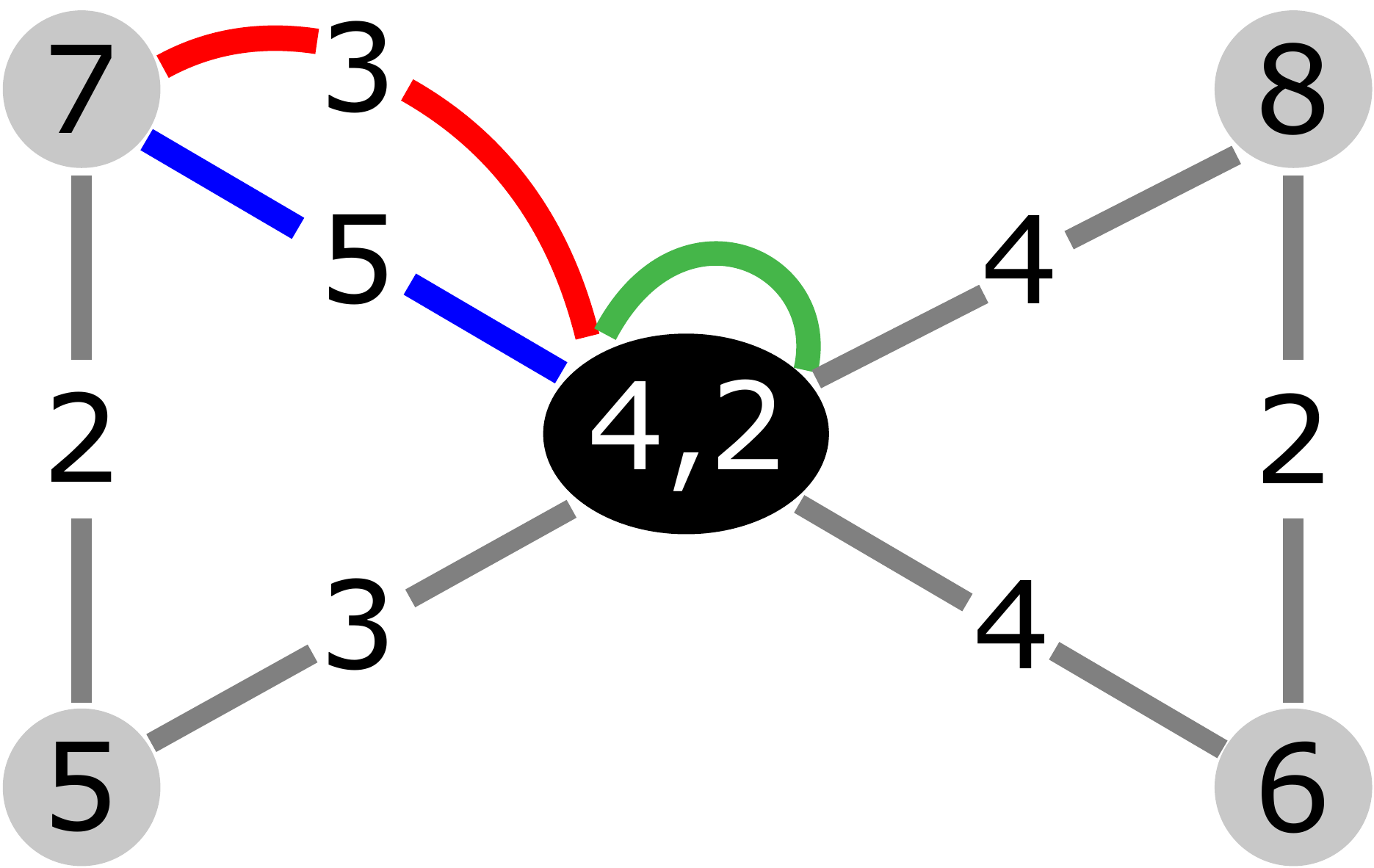} \\
 (a) & (b) \\
\includegraphics[width=\swidth]{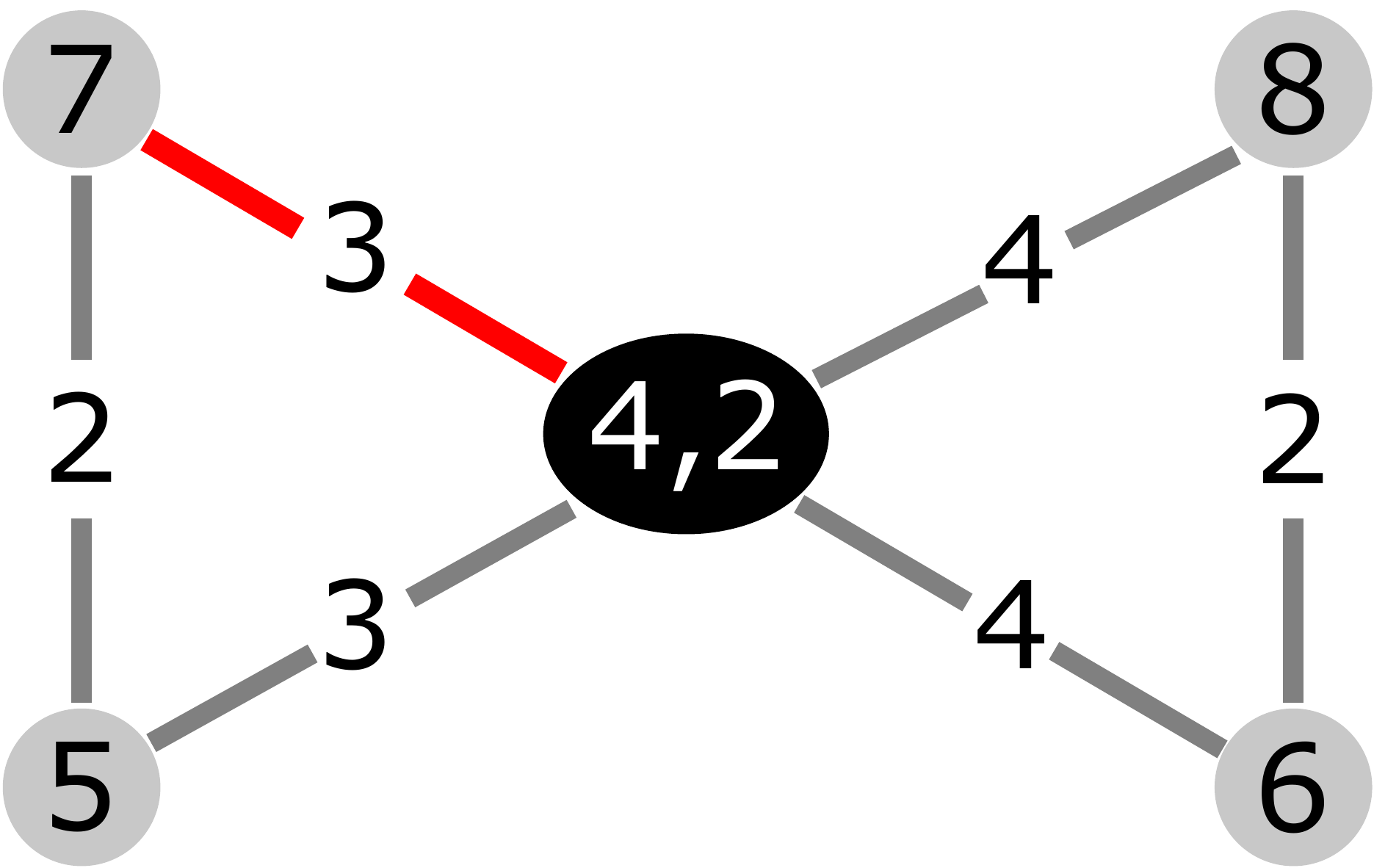} &
\includegraphics[width=\swidth]{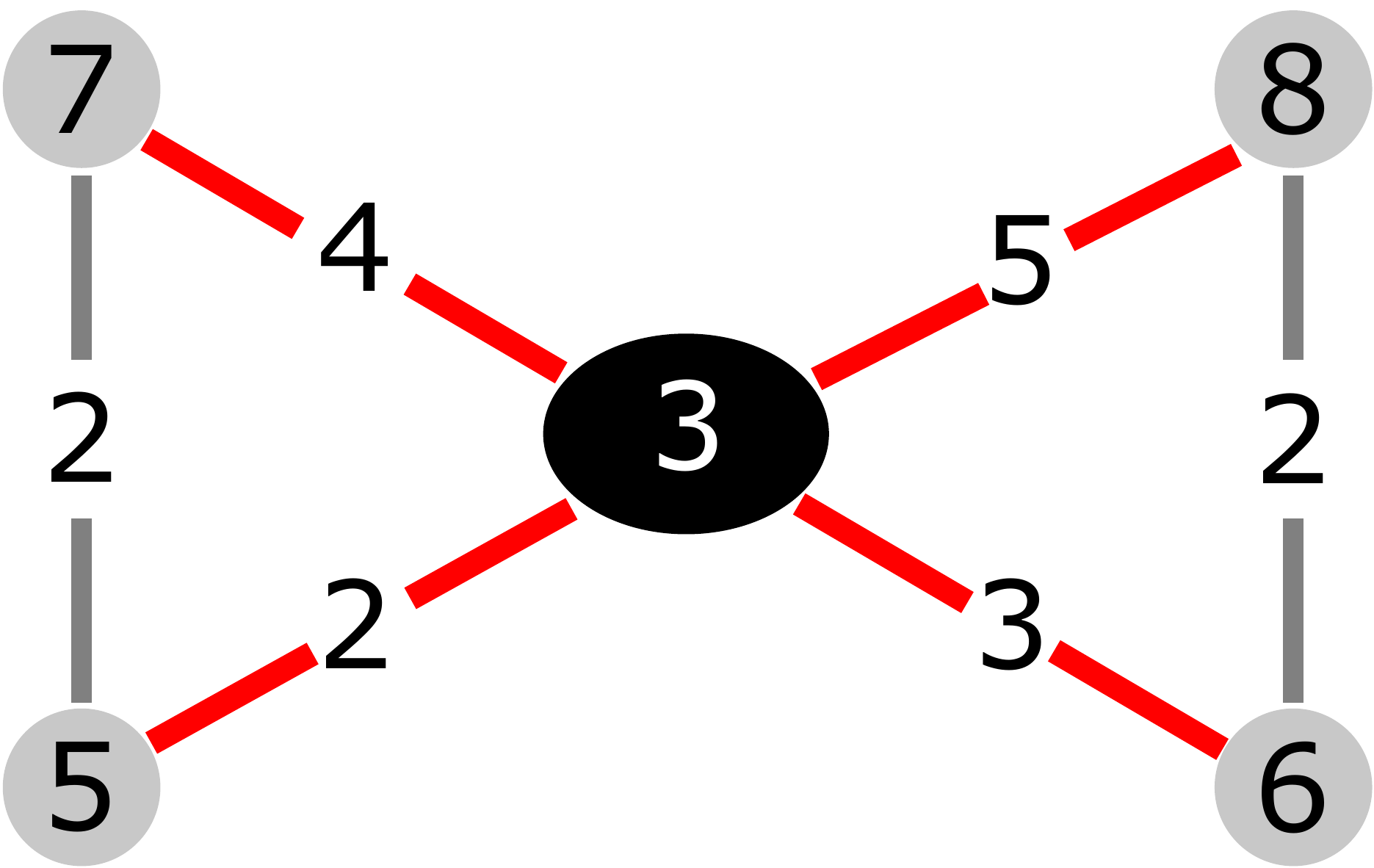} \\
(c) & (d) \\
\end{tabular}
\end{adjustbox}
& \hspace{-4mm}
\begin{adjustbox}{valign=c}
\begin{tabular}{c}
\setlength{\fboxrule}{1.0pt}
\setlength{\fboxsep}{0in}
\fbox{\includegraphics[width=0.25\linewidth]{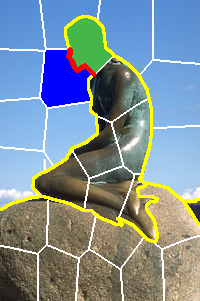}}
\vspace{1mm}\\ (e) \vspace{1mm}
\end{tabular}
\end{adjustbox}
\end{tabular}
\caption{\textbf{Illustration of edge contraction and feature aggregation.}
The number in each vertex representing its features.
Edge weights are computed by the absolute distance of their two ends.
(a)-(c) An edge contraction is performed between vertex $4$ and $2$.
After contracting an edge, the graph becomes a multigraph with a self-loop (green line) and parallel edges (blue and red line).
A flattening operation is followed by removing the self-loop and
replacing parallel edges by the lightest one (red line).
(d) Feature aggregation is performed after each iteration by gathering features from newly formed clusters and then updating edge weights (red lines).
(e) We use two kinds of features: color and edge confidence (yellow curves).
Our method explicitly maintains the connectivity of clusters so that the edge confidence (red curve) between two regions can be determined directly.
This is our advantage over SLIC that connectivity is enforced additionally.
}
\label{fig:con}
\end{figure}

The Bor\r{u}vka's algorithm can be applied to superpixel segmentation
directly.
However, a straightforward application of this algorithm does not
generate satisfactory results
in our experiment.
The issues stem from its greedy and local nature.
%
Recall that Equation \eqref{dist1} measures the distance
between two trees as their minimum edge weight.
This measurement ignores much information as for each tree only the
attribute of one vertex is used and thus it is sensitive to outliers.
In addition to linear time complexity, another advantage of the
Bor\r{u}vka's algorithm is that it can incorporate
multi-scale information within one \textit{unified} framework.
Since we obtain new clusters after each iteration, it is natural to
aggregate the attributes of each cluster and update the weights
connected to other clusters.
Figure \ref{fig:con}(d) illustrates this procedure.
After merging vertex $4$ and $2$, they become a supervertex with average value of $3$.
The self-loop is removed and parallel edges are replaced by one edge.
At the same time, weights of all edges connected to the supervertex are
updated according to the distance of aggregated attributes (red lines
in Figure \ref{fig:con}(d)).
Feature aggregation takes advantages of ``the wisdom of crowds'' rather
than only two vertices such that better performance can be expected.

This feature aggregation procedure is motivated by SLIC in which
centroids are updated by calculating the new means after each iteration.
Intuitively, proposed SH is more robust and efficient than SLIC.
First, a limited number of centroids are updated in the SLIC procedure.
That is for an individual pixel, its attributes are unchanged which could be an outlier.
SH treats all pixels as centriods and generates a hierarchy in a fine-to-coarse manner, therefore, SH is more robust than SLIC.
Second, though both SH and SLIC search limited regions (SH works on a planar graph and SLIC searches around predefined centers)
for nearest neighbor assignment in one iteration, the number of nodes in SH reduce geometrically while SLIC remains.
SLIC could cut down several iterations for efficiency but the accuracy also decreases.
In addition, the features to be used for clustering depend on the task and may not lie in the Euclidean space (e.g. edge confidence)
so that centroids can not be calculated simply by means.
Specified to the edge feature, our experiments (Section \ref{sec:experiment}) show that edge information is very useful to superpixels.
Our method explicitly maintains the connectivity of clusters so that the edge confidence between two regions can be determined directly as shown in Figure~\ref{fig:con}(e).
This is our advantage over SLIC-like methods that connectivity must be enforced additionally so that it's unclear how to integrate edge information into such procedure efficiently.

Incorporating edge confidence, our distance measure becomes
\begin{equation}
D(\mathcal{C}_1,\mathcal{C}_2) = d_c \times d_e,
\label{dist2}
\end{equation}
where $d_c$ and $d_e$ are color and edge distance, respectively.
The color distance is measured by the absolute difference of mean color.
However, mean color is not sufficient to represent superpixels as they become larger and larger.
For better performance, we measure color difference by the ${\chi ^2}$ distance of color histograms after $j$ iterations.
We use the structured forest edges (SFE)~\cite{DollarICCV13edges} to compute edge features and the distance is measured by the average edge confidence between the two regions (red curve in Figure~\ref{fig:con}(e)).

\section{Experiments}
\label{sec:experiment}

\def\minpagewidth{0.33\linewidth}
\def\marksize{1.5}
\begin{figure*}[!ht]
\begin{center}
\scalebox{0.8}{
\hspace{-2mm}\begin{minipage}[b]{\minpagewidth}
\centering
\begin{tikzpicture}[/pgfplots/width=1.0\linewidth,/pgfplots/height=0.9\linewidth]
  \begin{axis}[ymin=0.91,ymax=0.97,xmin=100,xmax=600,enlargelimits=false,
  title={BSDS500},
  title style={yshift=-1mm},
  xlabel={Number of Superpixles},
  ylabel={Achievable Seg Accuracy (ASA)},
  ylabel shift={-2pt},
  xtick={100,150,200,300,400,600},
  xticklabels={100,,200,300,400,600},
  ytick={0,0.01,...,1},
  yticklabel style={/pgf/number format/precision=2},
  font=\scriptsize,
  legend columns=-1,
  legend entries={$\mbox{SH}_E$-Our, $\mbox{SH}$-Our, LSC~\cite{LiC15}, SEEDS~\cite{van2012seeds}, ERS~\cite{liu2011entropy}, SLIC~\cite{achanta2012slic}, FH~\cite{felzenszwalb2004efficient}},
  legend to name=named,
  grid=both,
  grid style=dotted,
  major grid style={white!20!black},
  minor grid style={white!70!black},
  axis equal image=false]

    \addplot+[color_SH_E,solid,mark=none,ultra thick]                            table[x=num,y=ASA] {./fig/bench/BSDS500_test_SH_E_SE.txt};
    \label{curve:sh_e}
    \addplot+[color_SH_D ,solid,mark=o,mark size=\marksize,thick]                table[x=num,y=ASA] {./fig/bench/BSDS500_test_SH_D.txt};
    \label{curve:sh}
    \addplot+[color_LSC,solid,mark=x,mark size=\marksize,thick]                  table[x=num,y=ASA] {./fig/bench/BSDS500_test_LSC.txt};
    \label{curve:lsc}
    \addplot+[color_SEEDS,solid,mark=triangle,mark size=\marksize,thick]         table[x=num,y=ASA] {./fig/bench/BSDS500_test_SEEDS.txt};
    \label{curve:seeds}
    \addplot+[color_ERS,solid,mark=+,mark size=\marksize,thick]                  table[x=num,y=ASA] {./fig/bench/BSDS500_test_ERS.txt};
    \label{curve:ers}
    \addplot+[color_SLIC,solid,mark=diamond,mark size=\marksize,thick]           table[x=num,y=ASA] {./fig/bench/BSDS500_test_SLIC.txt};
    \label{curve:slic}
    \addplot+[color_FH,solid,mark=square,mark size=\marksize,thick]              table[x=num,y=ASA] {./fig/bench/BSDS500_test_FH.txt};
    \label{curve:fh}

  \end{axis}
  \node at (0.03*\linewidth,0.69*\linewidth) {(a)};
\end{tikzpicture}
\end{minipage}

\hspace{-4mm}\begin{minipage}[b]{\minpagewidth}
\centering
\begin{tikzpicture}[/pgfplots/width=1.0\linewidth,/pgfplots/height=0.9\linewidth]
  \begin{axis}[ymin=0.06,ymax=0.19,xmin=100,xmax=600,
  title={BSDS500},
  title style={yshift=-1mm},
  xlabel={Number of Superpixles},
  ylabel={Under-Seg Error (UE)},
  ylabel near ticks,
  ylabel shift={-2pt},
  yticklabel style={/pgf/number format/fixed,/pgf/number format/precision=2},
  xtick={100,150,200,300,400,600},
  xticklabels={100,,200,300,400,600},
  ytick={0,0.02,...,1},
  font=\scriptsize,
  grid=both,
  grid style=dotted,
  major grid style={white!20!black},
  minor grid style={white!70!black},
  axis equal image=false]
    \addplot+[color_SH_E,solid,mark=none,ultra thick]                            table[x=num,y=UE] {./fig/bench/BSDS500_test_SH_E_SE.txt};
    \addplot+[color_SH_D,solid,mark=o,mark size=\marksize,thick]                 table[x=num,y=UE] {./fig/bench/BSDS500_test_SH_D.txt};
    \addplot+[color_LSC,solid,mark=x,mark size=\marksize,thick]                  table[x=num,y=UE] {./fig/bench/BSDS500_test_LSC.txt};
    \addplot+[color_SEEDS,solid,mark=triangle,mark size=\marksize,thick]         table[x=num,y=UE] {./fig/bench/BSDS500_test_SEEDS.txt};
    \addplot+[color_ERS,solid,mark=+,mark size=\marksize,thick]                  table[x=num,y=UE] {./fig/bench/BSDS500_test_ERS.txt};
    \addplot+[color_SLIC,solid,mark=diamond,mark size=\marksize,thick]           table[x=num,y=UE] {./fig/bench/BSDS500_test_SLIC.txt};
    \addplot+[color_FH,solid,mark=square,mark size=\marksize,thick]              table[x=num,y=UE] {./fig/bench/BSDS500_test_FH.txt};

  \end{axis}
  \node at (0.03*\linewidth,0.69*\linewidth) {(b)};
\end{tikzpicture}
\end{minipage}

\hspace{-4mm}\begin{minipage}[b]{\minpagewidth}
\centering
\begin{tikzpicture}[/pgfplots/width=1.0\linewidth,/pgfplots/height=0.9\linewidth]
  \begin{axis}[ymin=0.5,ymax=0.92,xmin=100,xmax=600,
  title={BSDS500},
  title style={yshift=-1mm},
  xlabel={Number of Superpixles},
  ylabel={Boundary Recall (BR)},
  ylabel shift={-2pt},
  xtick={100,150,200,300,400,600},
  xticklabels={100,,200,300,400,600},
  ytick={0,0.1,...,1},
  font=\scriptsize,
  legend pos=outer north east,
  grid=both,
  grid style=dotted,
  major grid style={white!20!black},
  minor grid style={white!70!black},
  axis equal image=false]
    \addplot+[color_SH_E,solid,mark=none,ultra thick]                            table[x=num,y=BR] {./fig/bench/BSDS500_test_SH_E_SE.txt};
    \addplot+[color_SH_D,solid,mark=o,mark size=\marksize,thick]                 table[x=num,y=BR] {./fig/bench/BSDS500_test_SH_D.txt};
    \addplot+[color_LSC,solid,mark=x,mark size=\marksize,thick]                  table[x=num,y=BR] {./fig/bench/BSDS500_test_LSC.txt};
    \addplot+[color_SEEDS,solid,mark=triangle,mark size=\marksize,thick]         table[x=num,y=BR] {./fig/bench/BSDS500_test_SEEDS.txt};
    \addplot+[color_ERS,solid,mark=+,mark size=\marksize,thick]                  table[x=num,y=BR] {./fig/bench/BSDS500_test_ERS.txt};
    \addplot+[color_SLIC,solid,mark=diamond,mark size=\marksize,thick]           table[x=num,y=BR] {./fig/bench/BSDS500_test_SLIC.txt};
    \addplot+[color_FH,solid,mark=square,mark size=\marksize,thick]              table[x=num,y=BR] {./fig/bench/BSDS500_test_FH.txt};
  \end{axis}
  \node at (0.03*\linewidth,0.69*\linewidth) {(c)};
\end{tikzpicture}
\end{minipage}

\hspace{-4mm}\begin{minipage}[b]{\minpagewidth}
\centering
    \begin{tikzpicture}[/pgfplots/width=1.0\linewidth,/pgfplots/height=0.9\linewidth]
      \begin{axis}[
      xmin=100,xmax=600,ymin=0.00,ymax=0.26,
      xtick={100,150,200,300,400,600},
      xticklabels={100,,200,300,400,600},
      ytick={0.00,0.02,0.04,0.06,0.08,0.10,0.12,0.16,0.18,0.22,0.24},
      yticklabels={0.00,,0.04,,0.08,,0.12,0.30,,0.66,},
      yticklabel style={/pgf/number format/.cd,fixed,precision=2},
      ylabel shift={-2pt},
      xlabel=Number of Superpixels,
      ylabel=Time in Seconds,
      title={BSDS500},
      title style={yshift=-1mm},
      font=\scriptsize,
      grid=both,
      grid style=dotted,
      major grid style={white!20!black},
      minor grid style={white!70!black}]

      \addplot+[color_SH_E,solid,mark=none,ultra thick]                            table[x=num,y=time] {./fig/bench/Time_BSDS500_test_SH_E.txt};
      \addplot+[color_SH_D,solid,mark=o,mark size=\marksize,thick]                 table[x=num,y=time] {./fig/bench/Time_BSDS500_test_SH_D.txt};
      \addplot+[color_LSC,solid,mark=x,mark size=\marksize,thick]                  table[x=num,y=time] {./fig/bench/Time_BSDS500_test_LSC.txt};
      \addplot+[color_SEEDS,solid,mark=triangle,mark size=\marksize,thick]         table[x=num,y=time] {./fig/bench/Time_BSDS500_test_SEEDS.txt};
      \addplot+[color_ERS,solid,mark=+,mark size=\marksize,thick]                  table[x=num,y=time] {./fig/bench/Time_BSDS500_test_ERS.txt};
      \addplot+[color_SLIC,solid,mark=diamond,mark size=\marksize,thick]           table[x=num,y=time] {./fig/bench/Time_BSDS500_test_SLIC.txt};
      \addplot+[color_FH,solid,mark=square,mark size=\marksize,thick]              table[x=num,y=time] {./fig/bench/Time_BSDS500_test_FH.txt};

      \end{axis}
      \node at (0.03*\linewidth,0.69*\linewidth) {(d)};
    \end{tikzpicture}
  \end{minipage}

}

\vspace{1mm}
\scalebox{0.8}{
\hspace{-2mm}\begin{minipage}[b]{\minpagewidth}
\centering
\begin{tikzpicture}[/pgfplots/width=1.0\linewidth,/pgfplots/height=0.9\linewidth]
\begin{axis}[ymin=0.84,ymax=0.96,xmin=100,xmax=600,enlargelimits=false,
  title={Pascal SegVOC12},
  title style={yshift=-1mm},
  xlabel={Number of Superpixles},
  ylabel={Achievable Seg Accuracy (ASA)},
  ylabel shift={-2pt},
  xtick={100,200,300,400,500,600},
  ytick={0,0.02,...,1},
  yticklabel style={/pgf/number format/precision=2},
  font=\scriptsize,
  legend pos=south east,
  grid=both,
  grid style=dotted,
  major grid style={white!20!black},
  minor grid style={white!70!black},
  axis equal image=false]
    \addplot+[color_SH_E,solid,mark=none,ultra thick]                            table[x=num,y=ASA] {./fig/bench/Pascal_12_val_SH_E_SE.txt};
    \addplot+[color_SH_D,solid,mark=o,mark size=\marksize,thick]                 table[x=num,y=ASA] {./fig/bench/Pascal_12_val_SH_D.txt};
    \addplot+[color_LSC,solid,mark=x,mark size=\marksize,thick]                  table[x=num,y=ASA] {./fig/bench/Pascal_12_val_LSC.txt};
    \addplot+[color_ERS,solid,mark=+,mark size=\marksize,thick]                  table[x=num,y=ASA] {./fig/bench/Pascal_12_val_ERS.txt};
    \addplot+[color_SLIC,solid,mark=diamond,mark size=\marksize,thick]           table[x=num,y=ASA] {./fig/bench/Pascal_12_val_SLIC.txt};
    \addplot+[color_FH,solid,mark=square,mark size=\marksize,thick]              table[x=num,y=ASA] {./fig/bench/Pascal_12_val_FH.txt};

  \end{axis}
  \node at (0.03*\linewidth,0.69*\linewidth) {(e)};
\end{tikzpicture}
\end{minipage}

\hspace{-4mm}\begin{minipage}[b]{\minpagewidth}
\centering
\begin{tikzpicture}[/pgfplots/width=1.0\linewidth,/pgfplots/height=0.9\linewidth]
\begin{axis}[ymin=0.84,ymax=0.96,xmin=100,xmax=600,enlargelimits=false,
  title={SBD},
  title style={yshift=-1mm},
  xlabel={Number of Superpixles},
  ylabel={Achievable Seg Accuracy (ASA)},
  ylabel shift={-2pt},
  xtick={100,200,300,400,500,600},
  ytick={0,0.02,...,1},
  yticklabel style={/pgf/number format/precision=2},
  font=\scriptsize,
  legend pos=south east,
  grid=both,
  grid style=dotted,
  major grid style={white!20!black},
  minor grid style={white!70!black},
  axis equal image=false]

    \addplot+[color_SH_E,solid,mark=none,ultra thick]                            table[x=num,y=ASA] {./fig/bench/SBD_SH_E_SE.txt};
    \addplot+[color_LSC,solid,mark=x,mark size=\marksize,thick]                  table[x=num,y=ASA] {./fig/bench/SBD_LSC.txt};
    \addplot+[color_ERS,solid,mark=+,mark size=\marksize,thick]                  table[x=num,y=ASA] {./fig/bench/SBD_ERS.txt};
    \addplot+[color_SLIC,solid,mark=diamon,mark size=\marksize,thick]            table[x=num,y=ASA] {./fig/bench/SBD_SLIC.txt};
    \addplot+[color_FH,solid,mark=square,mark size=\marksize,thick]              table[x=num,y=ASA] {./fig/bench/SBD_FH.txt};
    \addplot+[color_SH_D,solid,mark=o,mark size=\marksize,thick]                 table[x=num,y=ASA] {./fig/bench/SBD_SH_D.txt};

  \end{axis}
  \node at (0.03*\linewidth,0.69*\linewidth) {(f)};
\end{tikzpicture}
\end{minipage}

\hspace{-4mm}\begin{minipage}[b]{\minpagewidth}
\centering
\begin{tikzpicture}[/pgfplots/width=1.0\linewidth,/pgfplots/height=0.9\linewidth]
\begin{axis}[ymin=0.80,ymax=0.92,xmin=100,xmax=600,enlargelimits=false,
  title={COCO14},
  title style={yshift=-1mm},
  xlabel={Number of Superpixles},
  ylabel={Achievable Seg Accuracy (ASA)},
  ylabel shift={-2pt},
  xtick={100,200,300,400,500,600},
  ytick={0,0.02,...,1},
  yticklabel style={/pgf/number format/precision=2},
  font=\scriptsize,
  legend pos=south east,
  grid=both,
  grid style=dotted,
  major grid style={white!20!black},
  minor grid style={white!70!black},
  axis equal image=false]
    \addplot+[color_SH_E,solid,mark=none,ultra thick]                            table[x=num,y=ASA] {./fig/bench/COCO_14_val_SH_E_SE.txt};
    \addplot+[color_SH_D,solid,mark=o,mark size=\marksize,thick]                 table[x=num,y=ASA] {./fig/bench/COCO_14_val_SH_D.txt};
    \addplot+[color_LSC,solid,mark=x,mark size=\marksize,thick]                  table[x=num,y=ASA] {./fig/bench/COCO_14_val_LSC.txt};
    \addplot+[color_ERS,solid,mark=+,mark size=\marksize,thick]                  table[x=num,y=ASA] {./fig/bench/COCO_14_val_ERS.txt};
    \addplot+[color_SLIC,solid,mark=diamond,mark size=\marksize,thick]           table[x=num,y=ASA] {./fig/bench/COCO_14_val_SLIC.txt};
    \addplot+[color_FH,solid,mark=square,mark size=\marksize,thick]              table[x=num,y=ASA] {./fig/bench/COCO_14_val_FH.txt};

  \end{axis}
  \node at (0.03*\linewidth,0.69*\linewidth) {(g)};
\end{tikzpicture}
\end{minipage}

\hspace{-4mm}\begin{minipage}[b]{\minpagewidth}
  \centering
    \begin{tikzpicture}[/pgfplots/width=1.0\linewidth,/pgfplots/height=0.9\linewidth]
      \begin{axis}[
      xmin=0.5,xmax=3.0,ymin=0.0,ymax=3.0,
      xtick={0.5,1.0,...,3.0},
      ytick={0.0,0.5,...,4.0},
      yticklabel style={/pgf/number format/.cd,fixed,precision=2},
      ylabel shift={-2pt},
      xlabel=Magapixels,
      ylabel=Time in Seconds,
      font=\scriptsize,
      grid=both,
      grid style=dotted,
      major grid style={white!20!black},
      minor grid style={white!70!black}]

      \addplot+[color_SH_E,solid,mark=none,ultra thick]                            table[x=num,y=time] {./fig/bench/Time_Image_Size_SH_E.txt};
      \addplot+[color_SH_D,solid,mark=o,mark size=\marksize,thick]                 table[x=num,y=time] {./fig/bench/Time_Image_Size_SH_D.txt};
      \addplot+[color_LSC,solid,mark=x,mark size=\marksize,thick]                  table[x=num,y=time] {./fig/bench/Time_Image_Size_LSC.txt};
      \addplot+[color_SEEDS,solid,mark=triangle,mark size=\marksize,thick]         table[x=num,y=time] {./fig/bench/Time_Image_Size_SEEDS.txt};
      \addplot+[color_SLIC,solid,mark=diamond,mark size=\marksize,thick]           table[x=num,y=time] {./fig/bench/Time_Image_Size_SLIC.txt};
      \addplot+[color_FH,solid,mark=square,mark size=\marksize,thick]              table[x=num,y=time] {./fig/bench/Time_Image_Size_FH.txt};

      \end{axis}
      \node at (0.03*\linewidth,0.69*\linewidth) {(h)};
    \end{tikzpicture}
  \end{minipage}

}

\vspace{2mm}\scriptsize{\ref{named}}\vspace{2mm}

\caption{\textbf{Segmentation accuracy and efficiency evaluation}: results on BSDS500, Pascal SegVOC12, SBD, and COCO14.
}
\label{fig:benchmark}
\end{center}
\end{figure*}
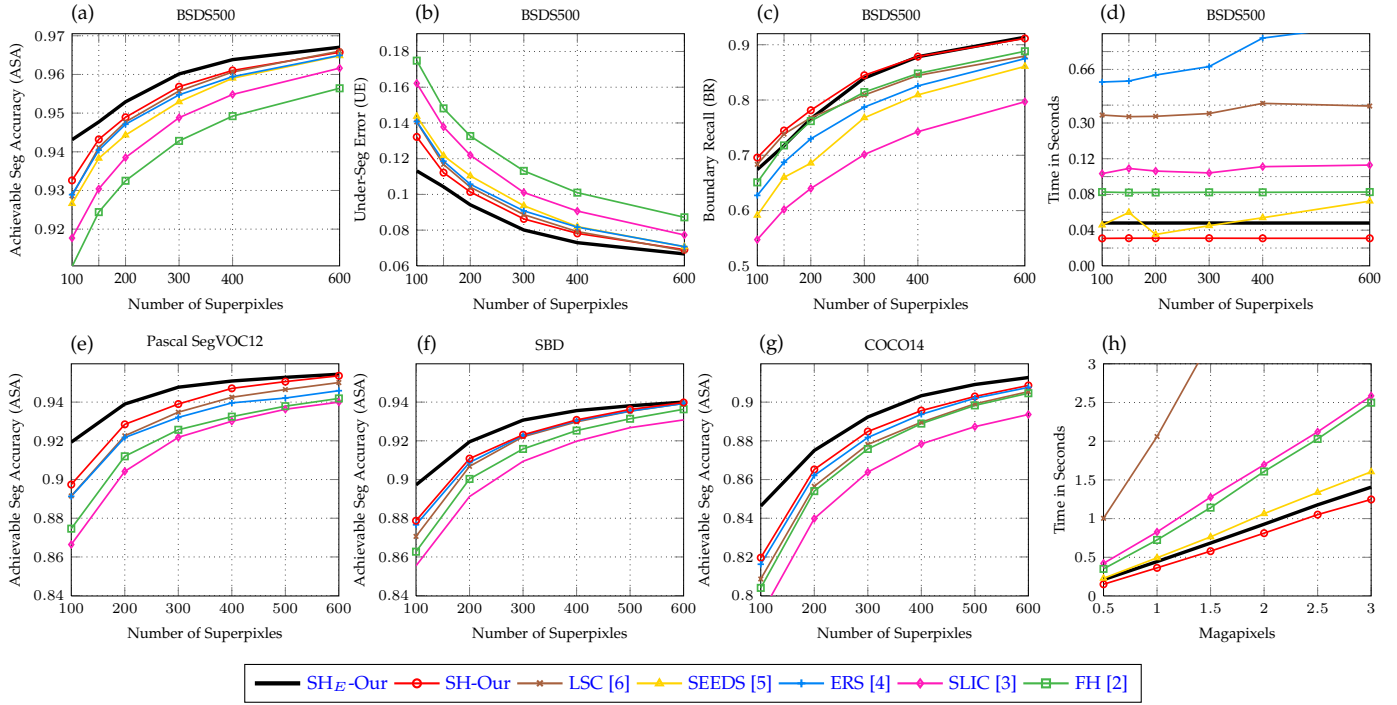

{\flushleft \textbf{Datasets and State of the Art.}}
We conduct experiments in the following four datasets: the Berkeley Segmentation Dataset (BSDS500)~\cite{pami-11-malik}, the segmentation challenge of Pascal 2012 Visual Object Classes (SegVOC12)~\cite{Everingham10}, the Berkeley Semantic Boundaries Dataset (SBD)~\cite{Hariharan2011}, and the Microsoft Common Objects in Context (COCO14)~\cite{Lin2014}.
BSDS500 is a common testbed for image segmentation with accurate annotated segments and boundaries. We perform throughout evaluation on it.
Pascal SegVOC12, SBD, and COCO14 are popular for object segmentation but do not have accurate boundaries. We test on them for object segmentation.

We compare our results against the following superpixel methods:
FH~\cite{felzenszwalb2004efficient}\footnote{\url{http://cs.brown.edu/~pff/segment/}},
SLIC~\cite{achanta2012slic}\footnote{\url{http://ivrl.epfl.ch/research/superpixels/}},
ERS~\cite{liu2011entropy}\footnote{\url{http://mingyuliu.net/}},
SEEDS~\cite{van2012seeds}\footnote{\url{http://www.mvdblive.org/seeds/}} and
LSC~\cite{LiC15}\footnote{\url{http://jschenthu.weebly.com/}}, using the implementations from the respective authors.
FH and SLIC are widely used in the literature because of their simplicity and efficiency.
ERS and LSC are considered as the state-of-the-art in terms of accuracy but are computationally expensive,
and SEEDS is the most efficient one among these five methods.
We evaluate two versions of proposed scheme:
\begin{itemize}
  \item \textbf{$\mbox{SH}$}  only uses color feature, and
  \item \textbf{$\mbox{SH}_E$} combines both color and edge features.
\end{itemize}

{\flushleft \textbf{Parameter Settings.}}
Our color difference is measured by the ${\chi ^2}$ distance of color histograms (equally divided into $k$ bins) after $j$ iterations.
These parameters are set based on a training database and fixed as $\{j, k\} = \{4, 20\}$ in our experiments.
For other methods, the default parameters published by the authors are used to ensure a fair comparison.
%
%

\subsection{Segmentation Accuracy}

{\flushleft \textbf{Benchmark Metrics.}}
We adopted the widely-used metrics to evaluate superpixel segmentation
methods including boundary recall, under-segmentation error \cite{neubert2012superpixel} and achievable segmentation accuracy \cite{liu2011entropy, van2012seeds}.


\begin{asparaitem}

\item{Achievable Segmentation Accuracy (ASA)}: gives the upper-bound segmentation accuracy.
It measures the fraction of ground truth segment that is correctly labeled by superpixels
\vspace{-1mm}
\begin{equation}
ASA(\mathcal{S}) = \frac{{\sum\nolimits_k {{{\max }_i}\left| {{s_k} \cap {g_i}} \right|} }}{{\sum\nolimits_i {{\left| {{g_i}} \right|}} }},
\vspace{-1mm}
\end{equation}
where $g_i$ is a ground truth segment, $s_k$ is a superpixel and $\left|  \cdot  \right|$ indicates the size of the segment.

\item{Under-segmentation Error (UE)}: compares
superpixel segment areas to measure to what extent superpixels cover
the ground truth segment border
\vspace{-1mm}
\begin{equation}
UE(\mathcal{S}) = \frac{{\sum\nolimits_i {\sum\nolimits_k {\min (\left| {{s_k} \cap {g_i}} \right|,\left| {{s_k} - {g_i}} \right|)} } }}{{\sum\nolimits_i {\left| {{g_i}} \right|} }}.
\vspace{-1mm}
\end{equation}

\item{Boundary Recall (BR)}: measures the percentage of
ground truth edges fall within superpixel boundaries with a tolerance
distance $\varepsilon = 2$.
Given a ground truth boundary union sets $\mathcal{B}(g)$ and the
superpixel boundary  sets $\mathcal{B}(s)$, the boundary recall of a segmentation $\mathcal{S}$ is defined by
    \begin{equation}
    BR(\mathcal{S}) = \frac{{TP(\mathcal{S})}}{{TP(\mathcal{S}) + FN(\mathcal{S})}},
    \end{equation}
    where $TP(\mathcal{S})$ is the number of boundary pixels in
    $\mathcal{B}(g)$ that fall within a boundary pixel
    $\mathcal{B}(s)$ in the range $\varepsilon$, and $FN(\mathcal{S})$ is
    the contrary case.

\end{asparaitem}

{\flushleft \textbf{Experimental Results.}}
Figure \ref{fig:benchmark}(a)-(c) present the quantitative evaluation results under the three metrics on the BSDS500.
%
As can be seen, the performance of proposed $\mbox{SH}$ and $\mbox{SH}_E$ methods~(\ref{curve:sh} and \ref{curve:sh_e}) are the highest under all three metrics.
$\mbox{SH}$ outperforms LSC~(\ref{curve:lsc}) and ERS~(\ref{curve:ers}) even though that SH generates all scales of superpixels simultaneously and is extremely faster.
With the assistance of SFE~\cite{DollarICCV13edges}, $\mbox{SH}_E$ outperforms others significantly.
The results are similar on other three object datasets, as shown in Figure \ref{fig:benchmark}(e)-(g).
The ASA on marked objects~(i.e. ignoring background) is employed for evaluation.
SegVOC12 has the most detailed annotation among these three object datasets and SH shows more advantages over LSC and ERS on it.
SLIC~(\ref{curve:slic}) performs unfavorably to others because of its regular sampling strategy (see Figure~\ref{fig:sh_slic}).
LSC also starts with regular sampling but maps pixels into a high dimensional feature space.
This helps LSC to capture the global image structure, but the segmentation is also suggestible by image quality (see Figure~\ref{fig:semantic_comparison}).
A comparison of SH with several edge detection methods is presented in Table~\ref{tab:edge}.
Recent edge detectors like HED~\cite{xie15hed} further improve segmentation accuracy.
We choose SFE in this paper because it has real-time performance~\cite{DollarICCV13edges} and efficiency is crucial for superpixels.

\begin{figure*}[!t]
\begin{center}
\def\sheight{0.119\linewidth}
\setlength{\fboxrule}{0.5pt}
\setlength{\fboxsep}{0in}
\def\arraystretch{0.5}
\renewcommand{\tabcolsep}{0.5 pt}
\begin{tabular}{cccccccccccc}

\fbox{\includegraphics[height=\sheight]{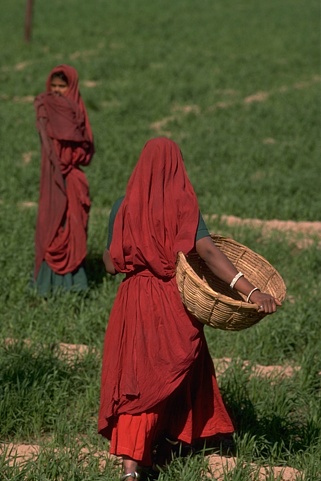}}       &
\fbox{\includegraphics[height=\sheight]{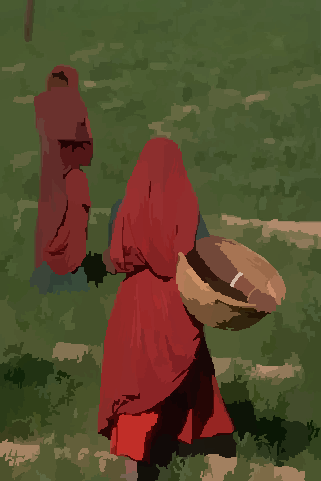}}   &
\fbox{\includegraphics[height=\sheight]{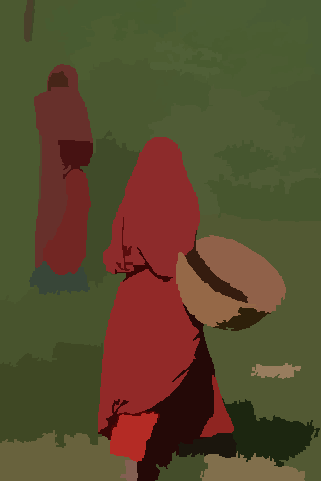}}    &
\fbox{\includegraphics[height=\sheight]{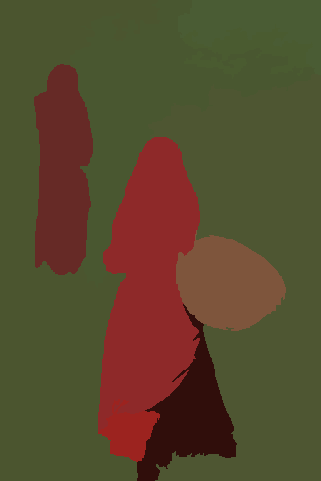}}    &
\fbox{\includegraphics[height=\sheight]{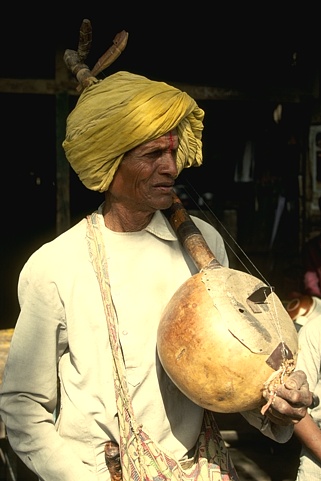}}      &
\fbox{\includegraphics[height=\sheight]{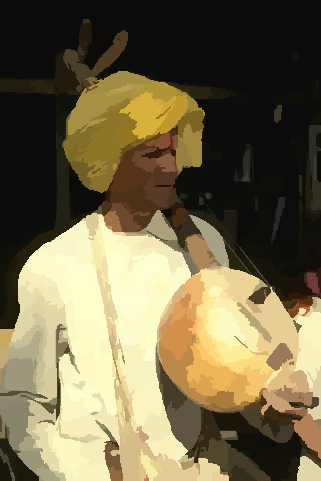}}  &
\fbox{\includegraphics[height=\sheight]{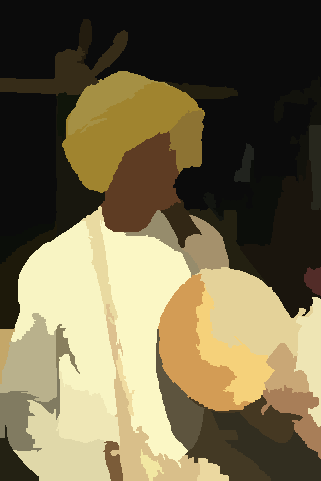}}   &
\fbox{\includegraphics[height=\sheight]{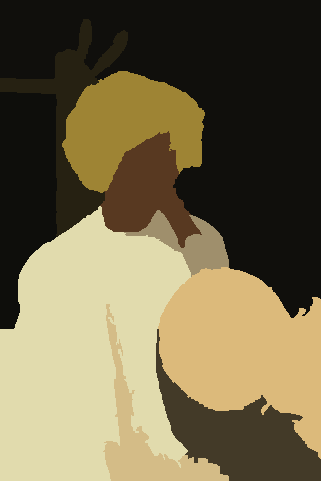}}   &
\fbox{\includegraphics[height=\sheight]{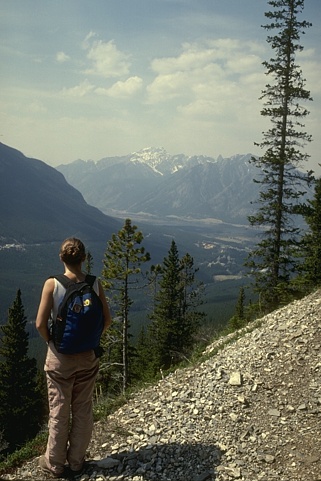}}      &
\fbox{\includegraphics[height=\sheight]{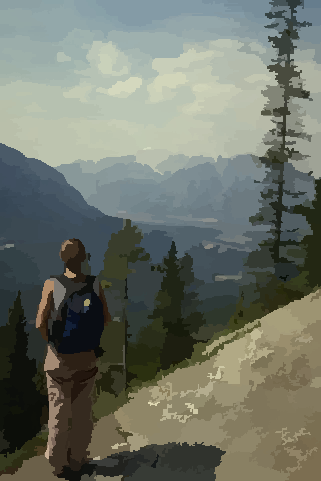}}  &
\fbox{\includegraphics[height=\sheight]{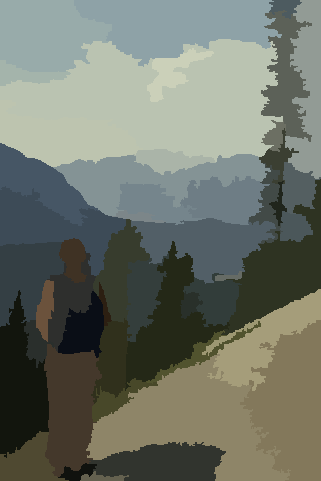}}   &
\fbox{\includegraphics[height=\sheight]{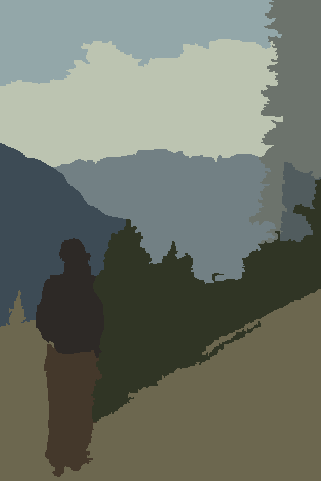}}   \\

\fbox{\includegraphics[height=\sheight]{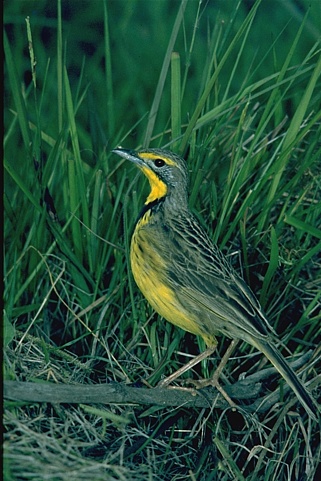}}      &
\fbox{\includegraphics[height=\sheight]{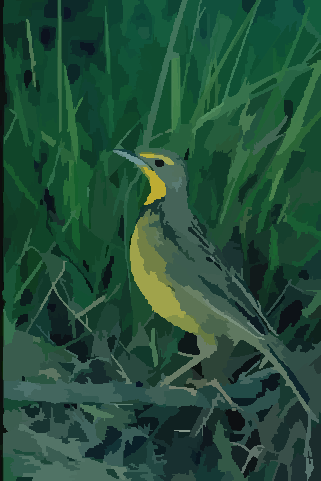}}  &
\fbox{\includegraphics[height=\sheight]{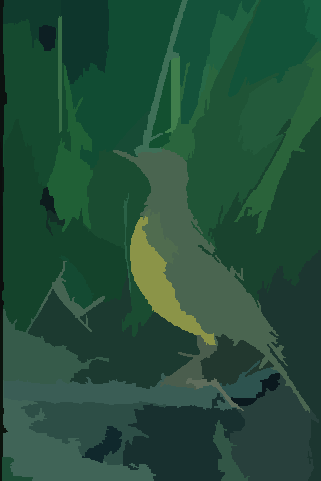}}   &
\fbox{\includegraphics[height=\sheight]{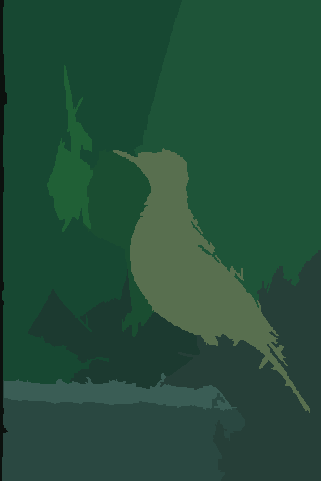}}   &
\fbox{\includegraphics[height=\sheight]{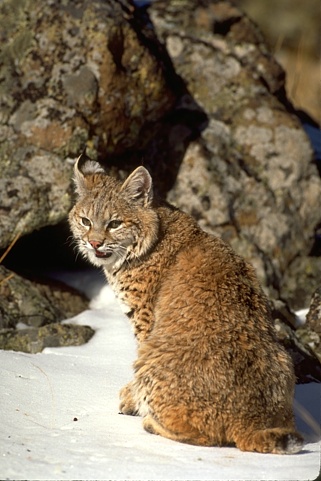}}      &
\fbox{\includegraphics[height=\sheight]{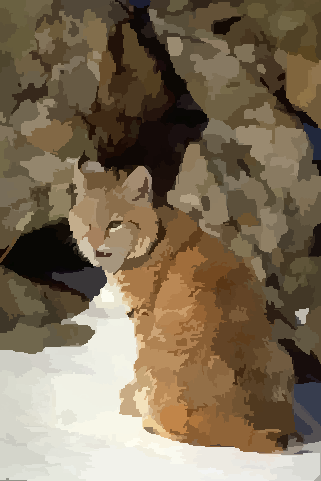}}  &
\fbox{\includegraphics[height=\sheight]{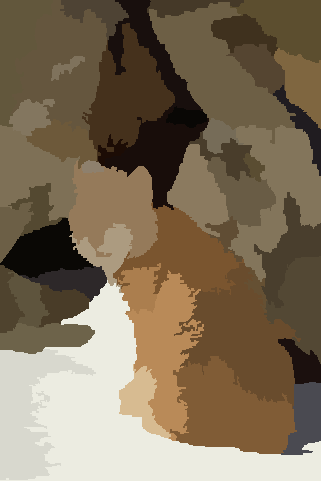}}   &
\fbox{\includegraphics[height=\sheight]{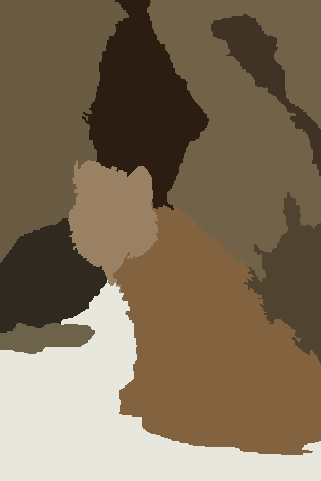}}   &
\fbox{\includegraphics[height=\sheight]{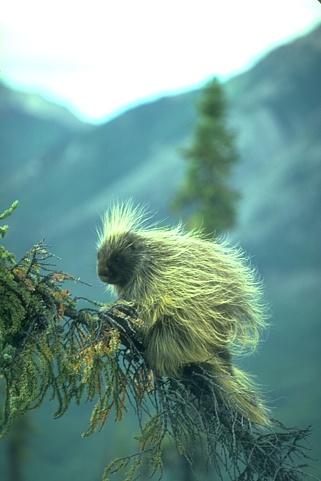}}      &
\fbox{\includegraphics[height=\sheight]{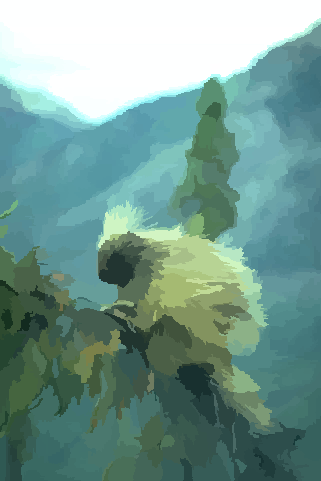}}  &
\fbox{\includegraphics[height=\sheight]{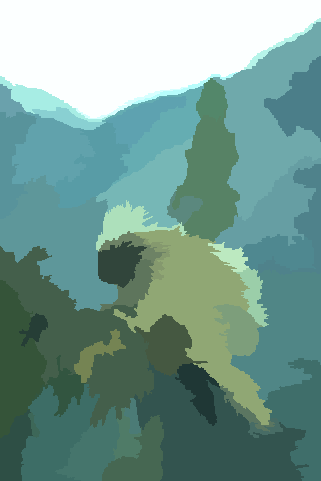}}   &
\fbox{\includegraphics[height=\sheight]{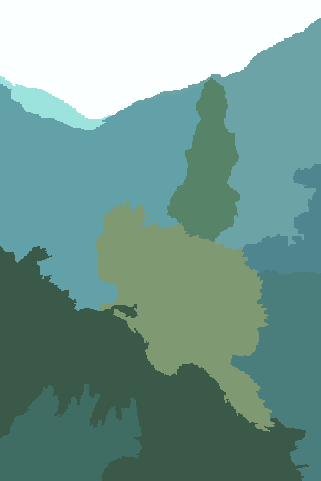}}   \\
\end{tabular}

\def\swidth{0.121\linewidth}
\setlength{\fboxrule}{0.5pt}
\setlength{\fboxsep}{0in}
\def\arraystretch{0.5}
\renewcommand{\tabcolsep}{0.5 pt}
\begin{tabular}{cccccccc}
\fbox{\includegraphics[width=\swidth]{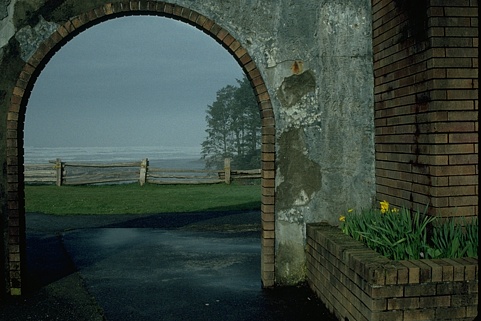}}        &
\fbox{\includegraphics[width=\swidth]{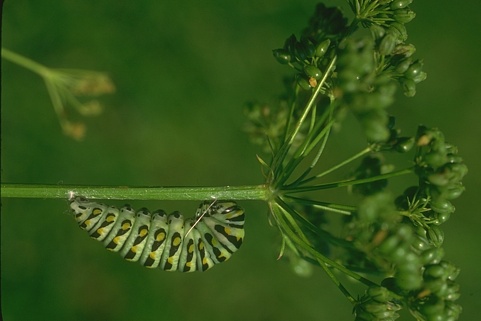}}       &
\fbox{\includegraphics[width=\swidth]{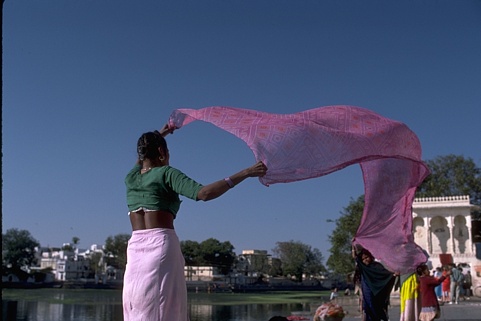}}       &
\fbox{\includegraphics[width=\swidth]{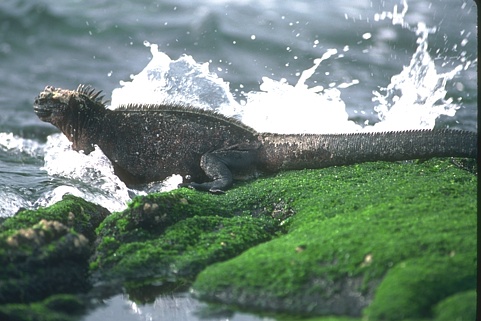}}      &
\fbox{\includegraphics[width=\swidth]{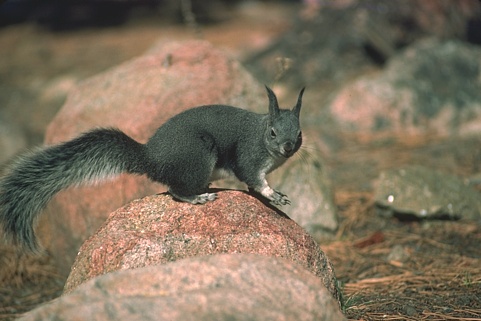}}      &
\fbox{\includegraphics[width=\swidth]{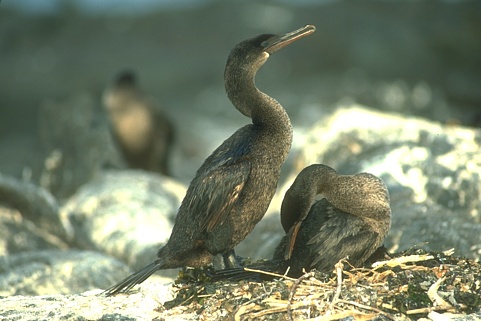}}      &
\fbox{\includegraphics[width=\swidth]{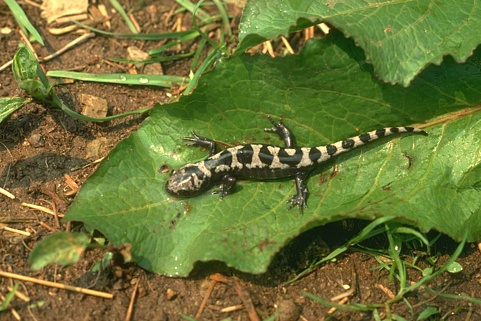}}      &
\fbox{\includegraphics[width=\swidth]{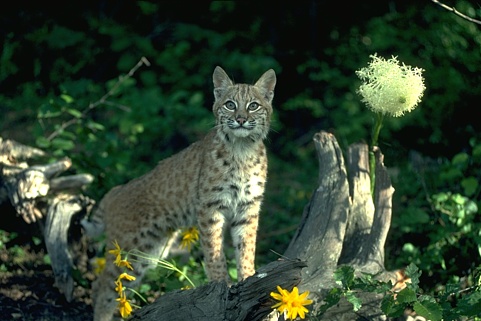}}      \\

\fbox{\includegraphics[width=\swidth]{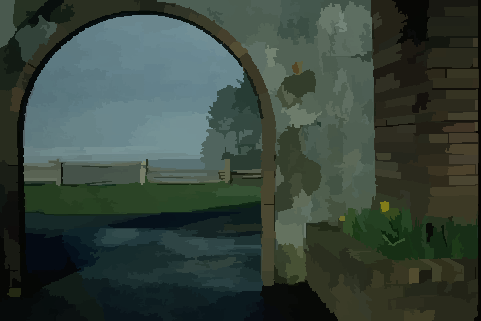}}      &
\fbox{\includegraphics[width=\swidth]{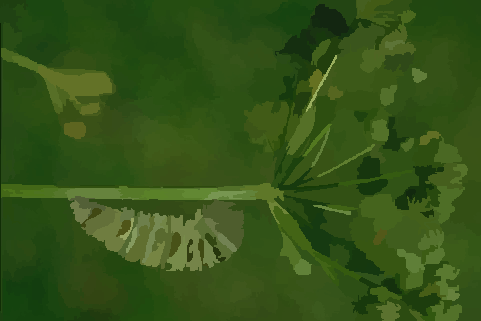}}     &
\fbox{\includegraphics[width=\swidth]{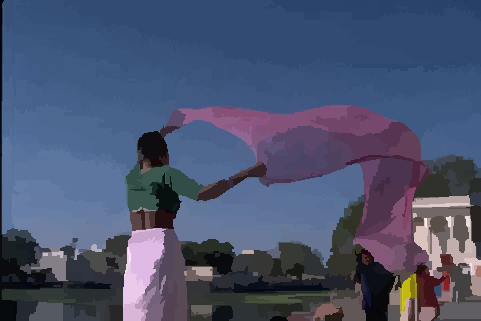}}     &
\fbox{\includegraphics[width=\swidth]{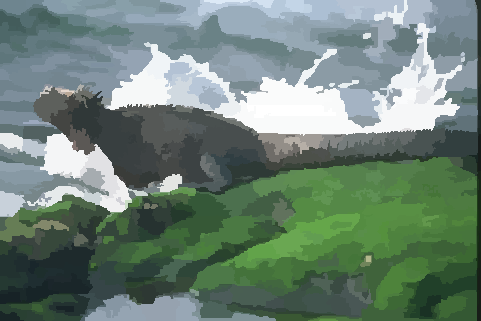}}    &
\fbox{\includegraphics[width=\swidth]{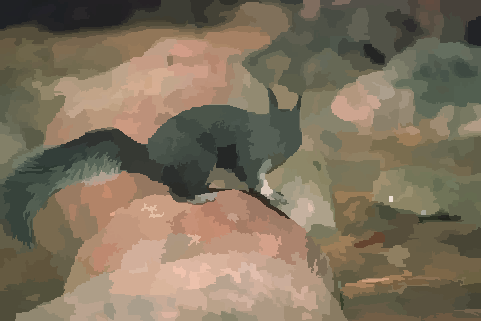}}    &
\fbox{\includegraphics[width=\swidth]{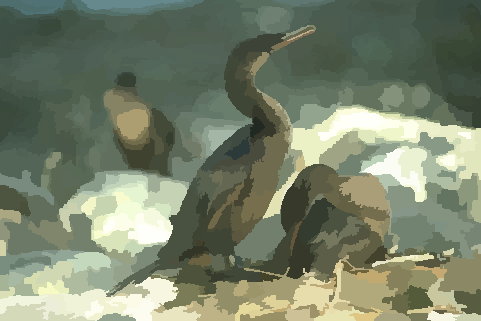}}    &
\fbox{\includegraphics[width=\swidth]{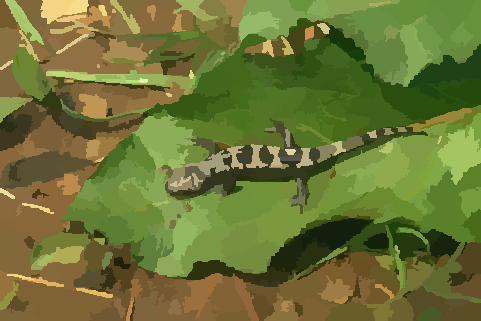}}    &
\fbox{\includegraphics[width=\swidth]{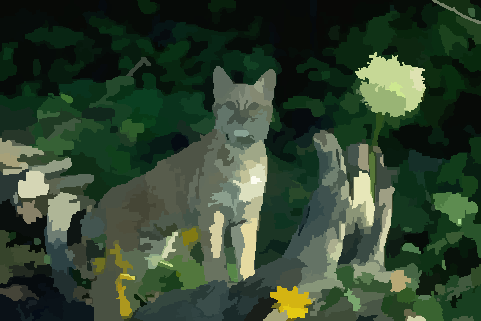}}   \\

\fbox{\includegraphics[width=\swidth]{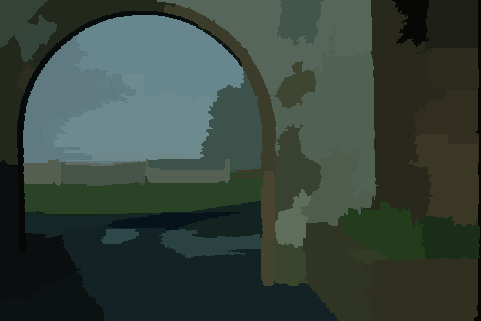}}      &
\fbox{\includegraphics[width=\swidth]{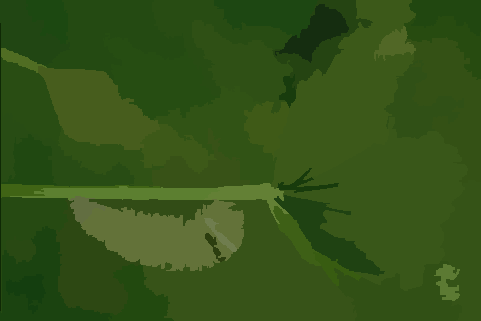}}     &
\fbox{\includegraphics[width=\swidth]{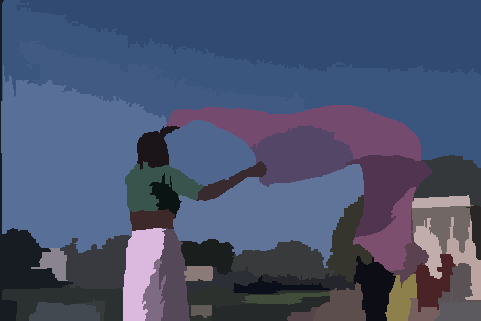}}     &
\fbox{\includegraphics[width=\swidth]{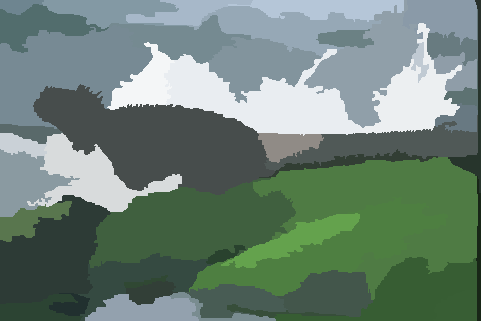}}    &
\fbox{\includegraphics[width=\swidth]{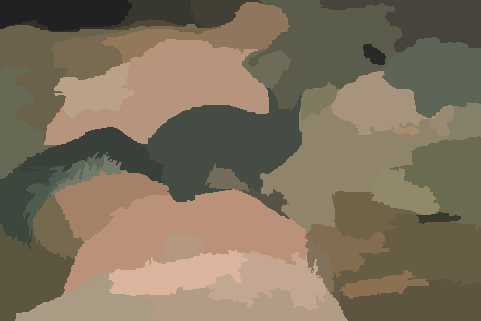}}    &
\fbox{\includegraphics[width=\swidth]{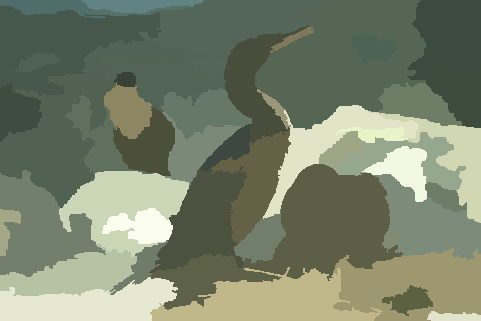}}    &
\fbox{\includegraphics[width=\swidth]{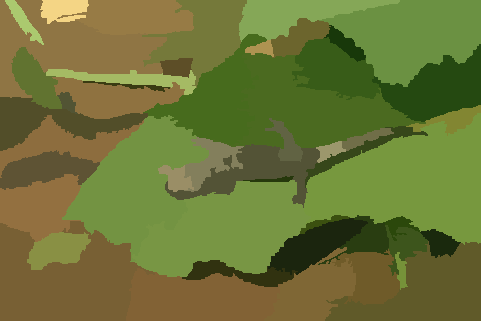}}    &
\fbox{\includegraphics[width=\swidth]{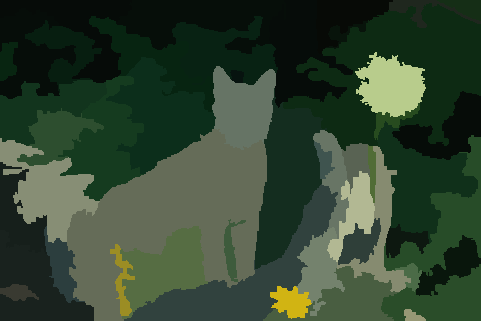}}   \\

\fbox{\includegraphics[width=\swidth]{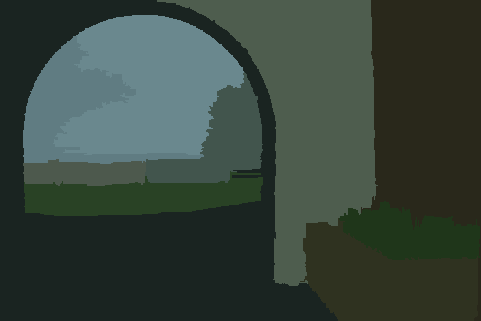}}      &
\fbox{\includegraphics[width=\swidth]{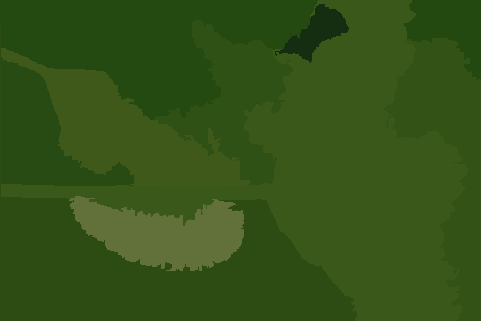}}     &
\fbox{\includegraphics[width=\swidth]{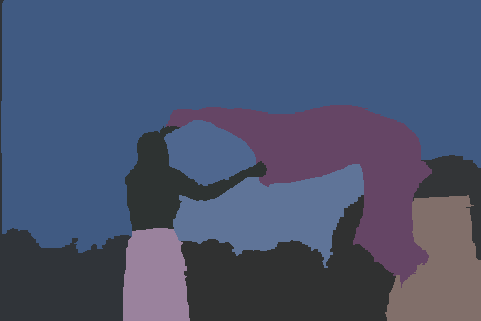}}     &
\fbox{\includegraphics[width=\swidth]{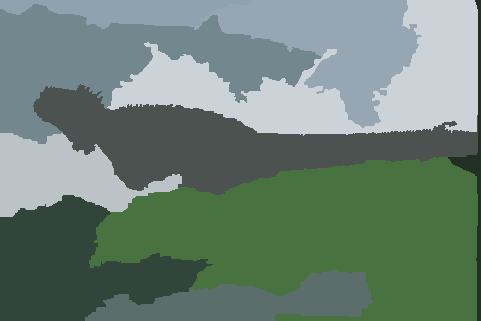}}    &
\fbox{\includegraphics[width=\swidth]{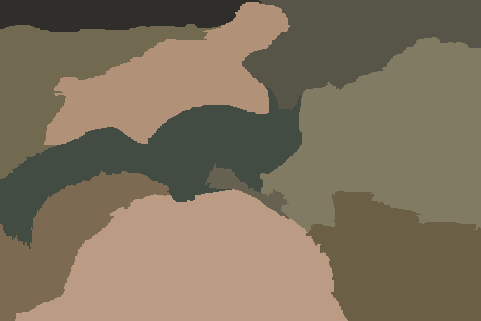}}    &
\fbox{\includegraphics[width=\swidth]{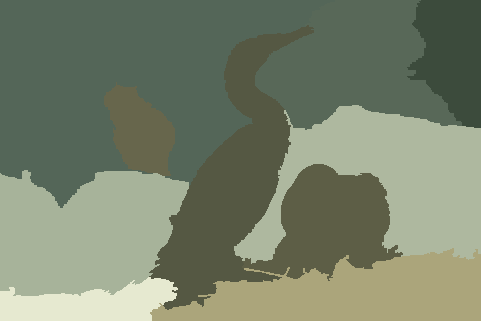}}    &
\fbox{\includegraphics[width=\swidth]{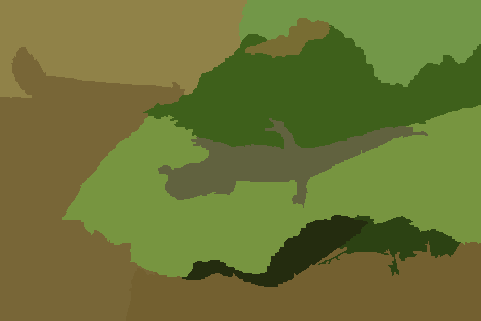}}    &
\fbox{\includegraphics[width=\swidth]{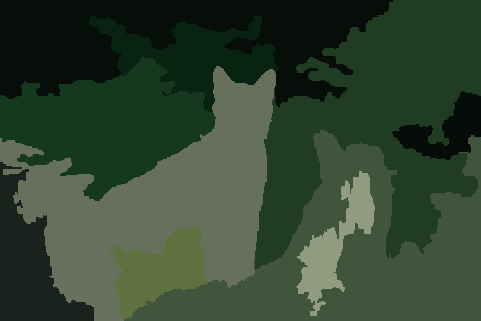}}   \\
\end{tabular}

\end{center}
\caption{\textbf{Super hierarchy on the BSDS500 test set}: input image, hierarchical segmentation with 600, 50 and 10 superpixels. }
\label{fig:super_hierarchy}
\end{figure*}

\begin{table}
\centering
\begin{tabular}{l|ccc}
\cmidrule[\heavyrulewidth](l{-2pt}){2-4}
    & mean ASA & mean UE & mean BR \\
\midrule
SH+SCG\,~\cite{ren_nips12}             & 95.4\% & 9.1\% & 78.4\% \\
SH+SFE\,\,\,~\cite{DollarICCV13edges}  & 95.5\% & 8.9\% & 79.8\% \\
SH+HED~\cite{xie15hed}                 & \textbf{95.7\%} & \textbf{8.6\%} & \textbf{80.2\%} \\
\bottomrule
\end{tabular}
\vspace{2mm}
\caption{\textbf{SH with several edge detectors}: results on BSDS500.}
\label{tab:edge}
\end{table}

\begin{figure}\scriptsize
\centering
\def\swidth{0.19\linewidth}
\setlength{\fboxrule}{0.5pt}
\setlength{\fboxsep}{0in}
\def\arraystretch{0.5}
\renewcommand{\tabcolsep}{0.5 pt}
\begin{tabular}{ccccc}
\fbox{\includegraphics[width=\swidth]{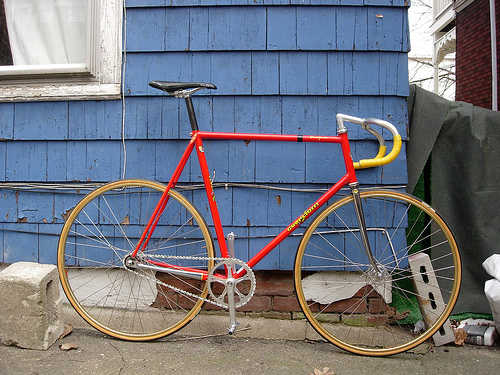}}                  &
\fbox{\includegraphics[width=\swidth]{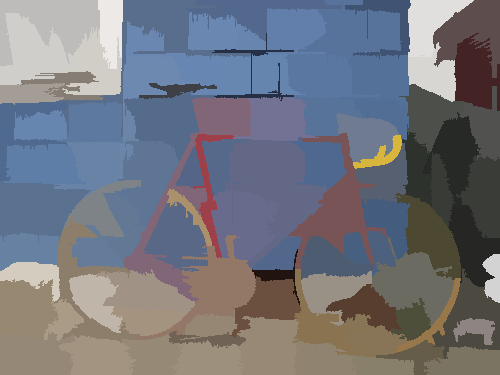}} &
\fbox{\includegraphics[width=\swidth]{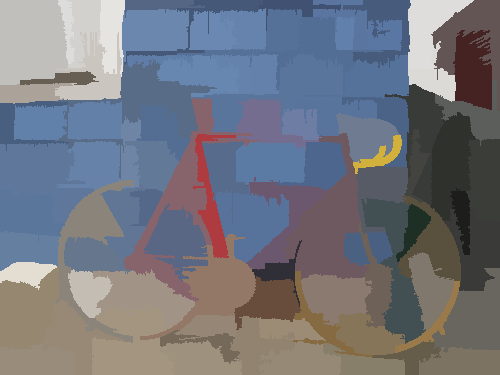}}  &
\fbox{\includegraphics[width=\swidth]{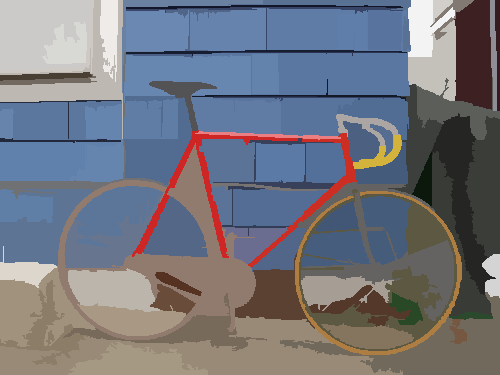}}         &
\fbox{\includegraphics[width=\swidth]{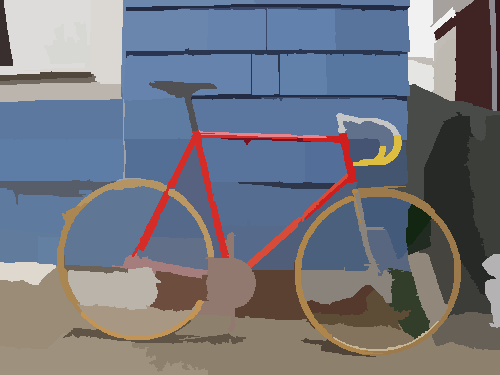}}         \\

\fbox{\includegraphics[width=\swidth]{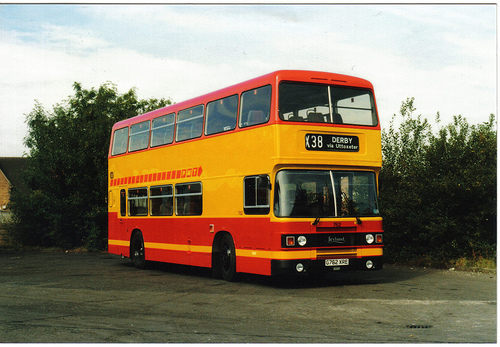}}                  &
\fbox{\includegraphics[width=\swidth]{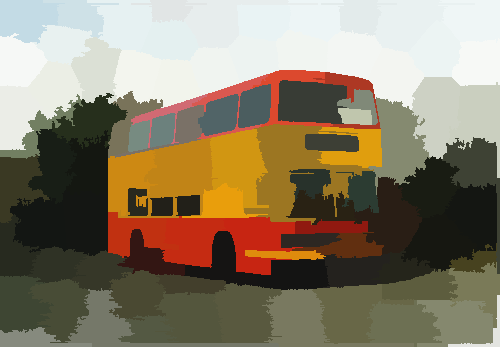}} &
\fbox{\includegraphics[width=\swidth]{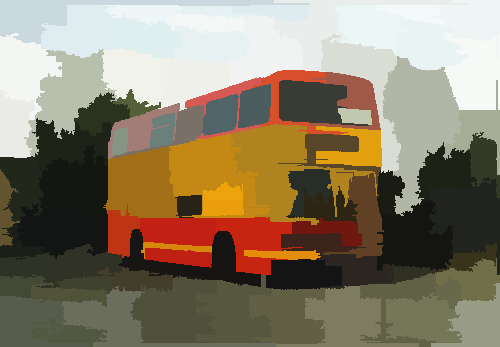}}  &
\fbox{\includegraphics[width=\swidth]{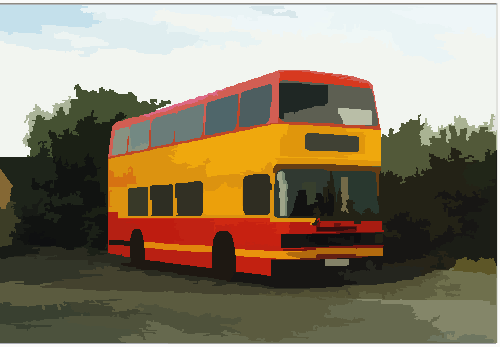}}         &
\fbox{\includegraphics[width=\swidth]{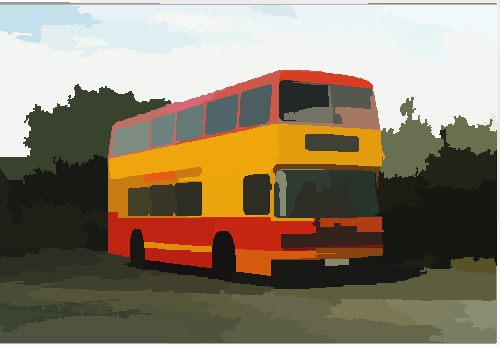}}         \\

\fbox{\includegraphics[width=\swidth]{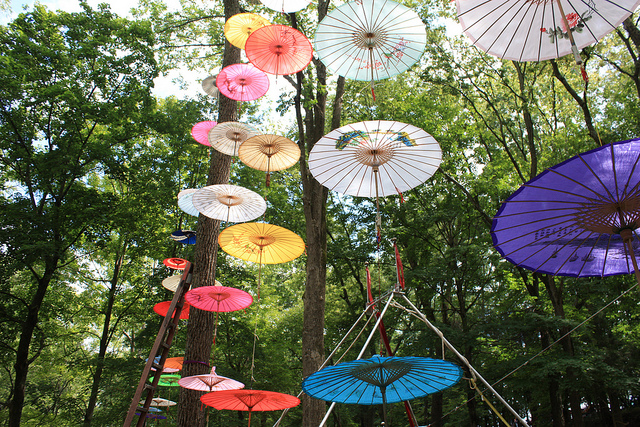}}                  &
\fbox{\includegraphics[width=\swidth]{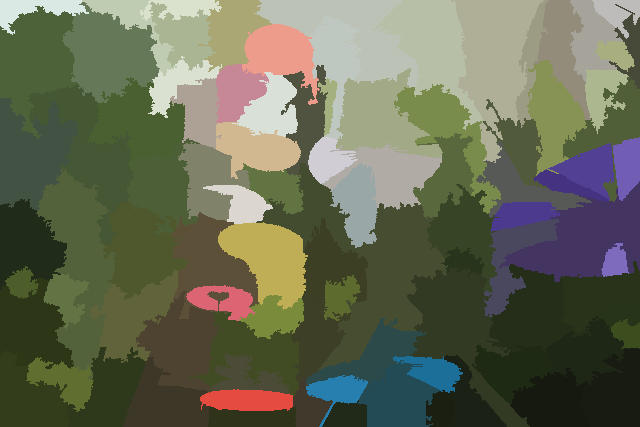}} &
\fbox{\includegraphics[width=\swidth]{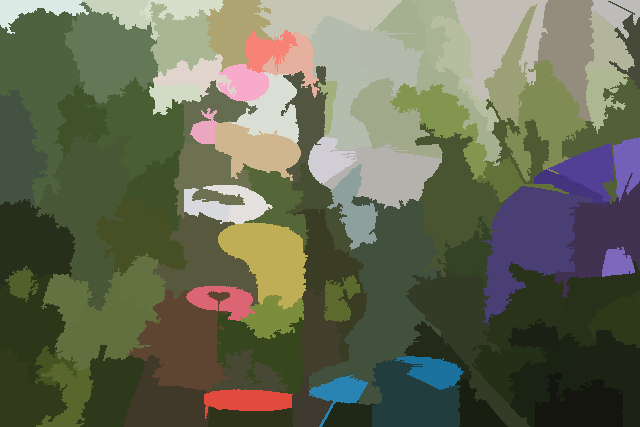}}  &
\fbox{\includegraphics[width=\swidth]{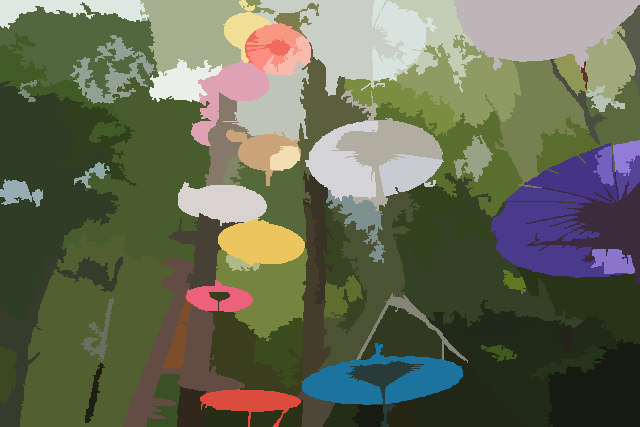}}         &
\fbox{\includegraphics[width=\swidth]{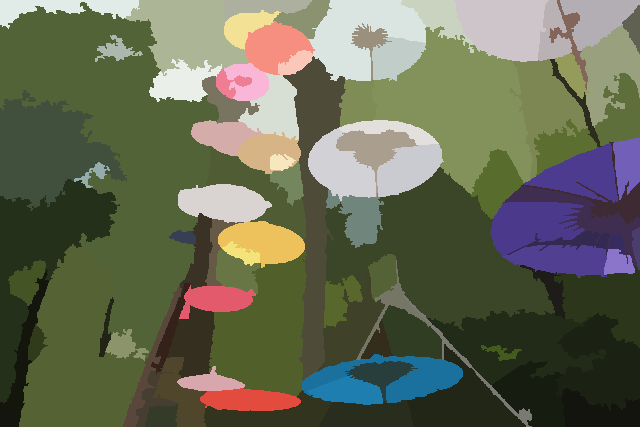}}         \\

\\ input & SLIC($m=10$) & SLIC($m=2$) & $\mbox{SH}$ & $\mbox{SH}_E$ \vspace{1mm}
\end{tabular}
\caption{\textbf{SH compared to SLIC with 100 superpixels}: results on Pascal segVOC12, SBD, and COCO14.
SLIC performs poorly on fine-structured objects, reducing the compact factor $m$ still get unsatisfactory results.
$\mbox{SH}_E$ improves $\mbox{SH}$ by leveraging edge features. It's unclear how to incorporate edges into SLIC.}
\label{fig:sh_slic}
\end{figure}

\begin{table*}[!tb]
\begin{center}
\scalebox{0.8}{
\def\arraystretch{1.3}
\renewcommand{\tabcolsep}{3.5 pt}
\begin{tabular}{l|ccccccccccccccccccccc|ccc}
\midrule[\heavyrulewidth]
 & \rotatebox{0}{Build.} & \rotatebox{0}{Grass} & \rotatebox{0}{Tree} & \rotatebox{0}{Cow} & \rotatebox{0}{Sheep} & \rotatebox{0}{Sky} & \rotatebox{0}{Air.} & \rotatebox{0}{Water} & \rotatebox{0}{Face} & \rotatebox{0}{Car} & \rotatebox{0}{Bicycle} & \rotatebox{0}{Flower} & \rotatebox{0}{Sign} & \rotatebox{0}{Bird} & \rotatebox{0}{Book} & \rotatebox{0}{Chair} & \rotatebox{0}{Road} & \rotatebox{0}{Cat} & \rotatebox{0}{Dog} & \rotatebox{0}{Body} & \rotatebox{0}{Boat} & \rotatebox{0}{\textbf{Global}} & \rotatebox{0}{\textbf{Average}}\\
\hline
FH\cite{felzenszwalb2004efficient}  & 77.9 & 92.4 & 88.1 & 75.8 & 75.4 & 89.3	& 48.5  & 58.5 & 82.6	& 60.0 & 82.2 & 45.0 & 72.6	& 15.2  & 88.0 & 45.0	& 85.0 & 46.3 & 36.6 & 70.3	& 04.7  & 78.1\% & 63.8\% \\
SLIC\cite{achanta2012slic}          & 78.0 & 93.6	& 83.8 & 87.9 & 74.4 & 92.6 & 46.8  & 68.6 & 84.6	& 57.7 & 76.5 & 67.3 & 57.9	& 19.3  & 92.5 & 40.6	& 82.9 & 60.8 & 42.9 & 57.6	& 07.5  & 79.5\% & 65.4\% \\
ERS\cite{liu2011entropy}            & 78.9 & 92.5	& 85.5 & 83.6 & 59.1 & 95.5	& 68.6  & 67.9 & 86.4	& 54.7 & 74.1 & 53.1 & 69.8	& 25.0  & 93.2 & 37.0	& 84.1 & 51.3 & 48.6 & 52.7	& 13.8  & 79.6\% & 65.5\% \\
SEEDS\cite{van2012seeds}            & 77.1 & 92.7 & 88.3 & 81.8 & 71.8 & 95.3 & 57.2  & 70.2 & 82.5 & 53.7 & 76.6 & 63.5 & 67.0 & 22.8  & 94.3 & 38.4 & 85.5 & 50.0 & 48.0 & 55.8 & 10.4  & 80.3\% & 65.8\% \\
LSC\cite{LiC15}  & 81.5 & 93.4 & 84.4 & 84.7 & 78.0 & 93.4 & 46.4 & 73.2 & 83.7 & 52.7 & 77.4 & 73.8 & 60.0 & 16.3 & 92.5 & 25.2 & 85.9 & 51.5 & 48.1 & 57.8 & 11.4 & \textbf{80.5\%} & 65.3\% \\
SH                                  & 80.6 & 92.9 & 84.9 & 86.8 & 70.6 & 92.1 & 58.0  & 69.0 & 83.2 & 59.3 & 80.0 & 65.5 & 79.3 & 14.2  & 86.4 & 42.0 & 85.5 & 47.3 & 52.1 & 58.3 & 10.1  & 80.4\% & \textbf{66.6\%} \\

\midrule[\heavyrulewidth]
\end{tabular}}
\end{center}
\caption{\textbf{Semantic segmentation accuracy.} Using the method of~\cite{gould2008multi} on the MSRC-21~\cite{shotton2006textonboost}. The global score gives the percentage of correctly classified pixels and the average score provides the per-class average \cite{gonfaus2010harmony}.
The global scores of SEEDS, LSC, and proposed SH are similar while SH improves the per-class accuracy significantly. }
\label{tab:semantic}
\end{table*}

\subsection{Computational Efficiency}

\begin{figure}\scriptsize
\begin{center}
\def\arraystretch{0.6}
\renewcommand{\tabcolsep}{0.5pt}
\setlength{\fboxrule}{0.5pt}
\setlength{\fboxsep}{0pt}
\def\swidth{0.24\linewidth}
\begin{tabular}{cccc}

\fbox{\includegraphics[width=\swidth]{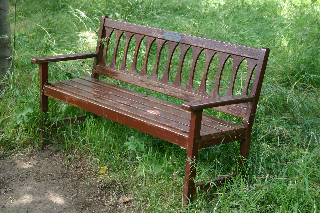}}               &
\fbox{\includegraphics[width=\swidth]{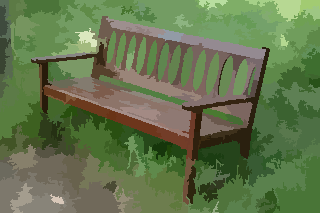}}     &
\fbox{\includegraphics[width=\swidth]{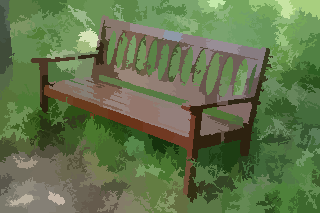}}      &
\fbox{\includegraphics[width=\swidth]{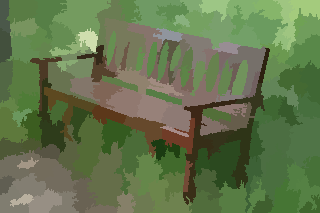}}     \\
\fbox{\begin{overpic}[width=\swidth,unit=1mm]{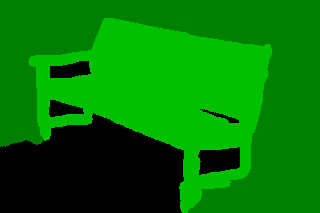}
\put(30,40){\color{white}{ chair}}
\put(75, 5){\color{white}{ grass}}
\put(8,  8){\color{white}{ void}}
\end{overpic}} &


\fbox{\includegraphics[width=\swidth]{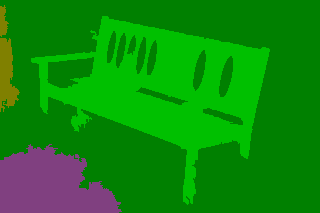}} &
\fbox{\includegraphics[width=\swidth]{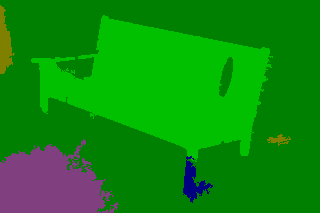}} &
\fbox{\includegraphics[width=\swidth]{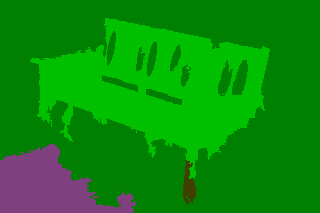}} \\

\\(a) input/GT & (b) SH & (c) LSC\cite{LiC15} & (d) SLIC\cite{achanta2012slic} \\
\end{tabular}
\end{center}

\caption{\textbf{Semantic segmentation examples.} \emph{Top}: input image and segmentation with 200 superpixels. \emph{Bottom}: ground truth (GT) and classification results for the method of \cite{gould2008multi} with different superpixel algorithms.
}
\label{fig:semantic_comparison}
\end{figure}

We evaluate the computational efficiency in two respects.
The run-time w.r.t the  number of superpixels and image size.
All algorithms are tested on a single 3.4 GHz i7 CPU.
We first analyze the complexity of each algorithm in theory and then
present empirical results.

{\flushleft \textbf{Theoretical Complexity.}}
FH uses Kruskal's algorithm \cite{west2001introduction} to grow region that runs in $O(n \log n)$ time with low constant factors.
SLIC is a $k$-means clustering procedure with
constrained search region which runs in $O(n)$ time of each iteration
but needs several iterations to convergence.
SEEDS maximizes its energy function via hill-climbing
optimization at two levels of granularity: pixel-level and
block-level.
The run-time of block-level optimization depends on the number of
superpixels.
ERS builds a submodular and monotonic objective function
that can be optimized by lazy greedy.
%
The worst case complexity of lazy greedy algorithms is $O(n^2 \log
n)$ while \cite{liu2011entropy} claims that on average the complexity
of ERS approximates $O(n \log n)$.
LSC shares a similar framework with SLIC and the complexity is also $O(n)$.
LSC is apparently slower than SLIC because it works in high dimensional feature
space and conducts more iterations in order to achieve higher segmentation accuracy.
As analyzed in Section \ref{sec:complexity}, the computational complexity of the proposed SH method is $O(n)$.
Compared to other linear time methods,
our contribution and advantage is that proposed method has $O(1)$ complexity to generate $m$ scales of superpixels while other methods are $O(m)$.

{\flushleft \textbf{Experimental Results.}}
Figure \ref{fig:benchmark}(d) shows the run-time with increasing number
of superpixels on the BSDS500.
%
%
The run-time of proposed SH~(\ref{curve:sh}) and FH~(\ref{curve:fh}) is independent of the amount of superpixels.
The run-time of SLIC~(\ref{curve:slic}) and LSC~(\ref{curve:lsc}) is a bit unstable but constant in general.
LSC is 10$\times$ slower than SH.
SEEDS~(\ref{curve:seeds}) varies significantly and the worst case here is twice slower than the fastest.
ERS~(\ref{curve:ers}) is 20$\times$ slower than SH and the run-time increases w.r.t. the number of superpixels.
Figure \ref{fig:benchmark}(h) shows the run-time w.r.t. image size.
Every set has 10 images and we report the average time.
The average size of superpixels is 1024 for all image sets and all
algorithms.
The result with the ERS algorithm is not plotted due to its high
computational cost.
The FH, SLIC, SEEDS, LSC, and the proposed SH methods all run in time near linear in image size in practice.

\begin{figure}\scriptsize
\begin{center}

\def\arraystretch{0.6}
\renewcommand{\tabcolsep}{0.5pt}
\setlength{\fboxrule}{0.5pt}
\setlength{\fboxsep}{0in}
\def\swidth{0.19\linewidth}
\begin{tabular}{ccccc}
\fbox{\includegraphics[width=\swidth]{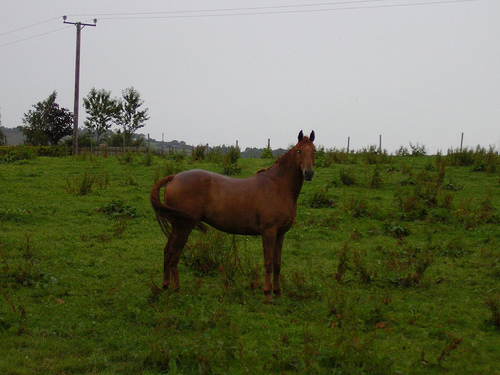}}           &
\fbox{\includegraphics[width=\swidth]{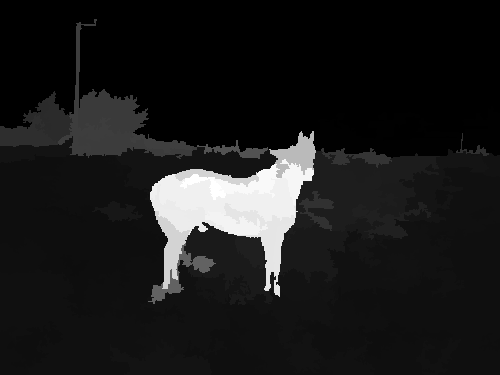}}  &
\fbox{\includegraphics[width=\swidth]{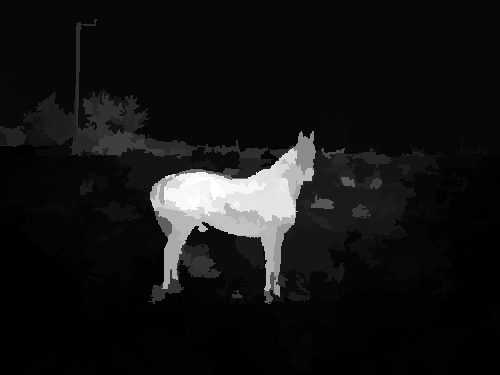}}  &
\fbox{\includegraphics[width=\swidth]{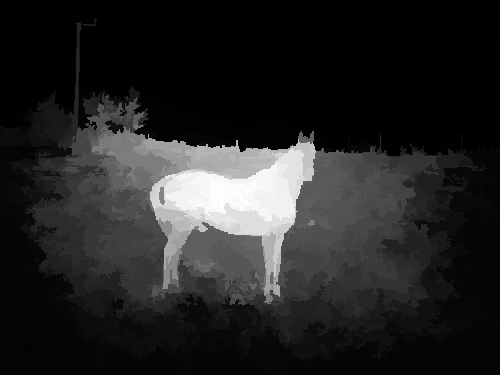}} &
\fbox{\includegraphics[width=\swidth]{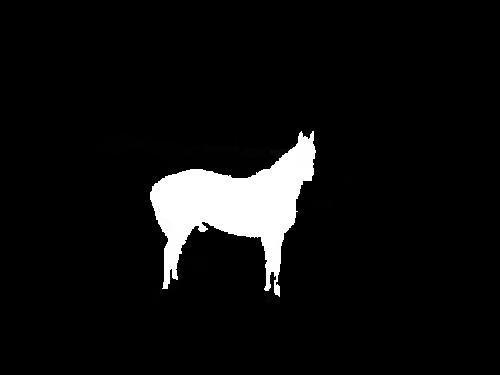}}       \\

\fbox{\includegraphics[width=\swidth]{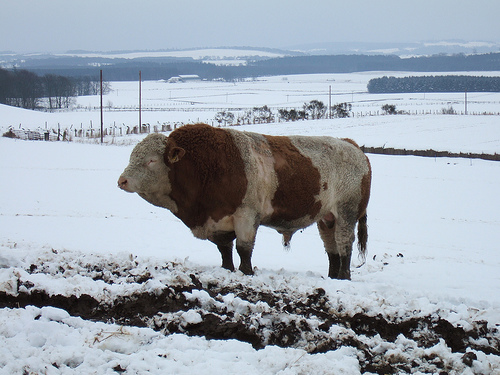}}           &
\fbox{\includegraphics[width=\swidth]{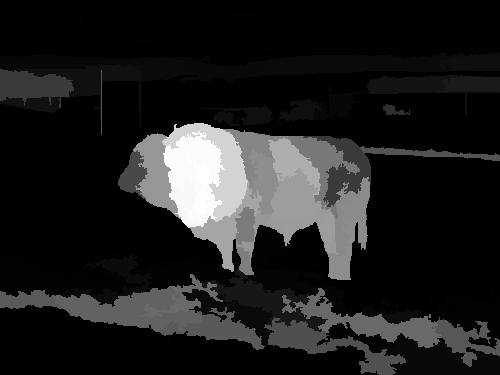}}  &
\fbox{\includegraphics[width=\swidth]{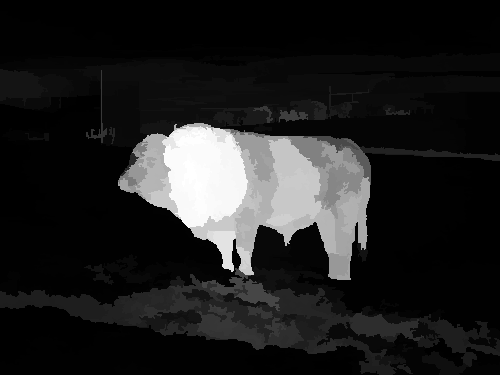}}  &
\fbox{\includegraphics[width=\swidth]{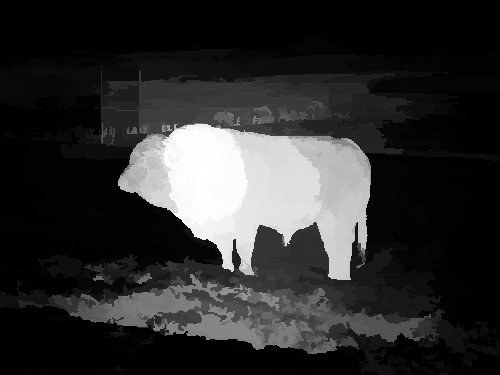}} &
\fbox{\includegraphics[width=\swidth]{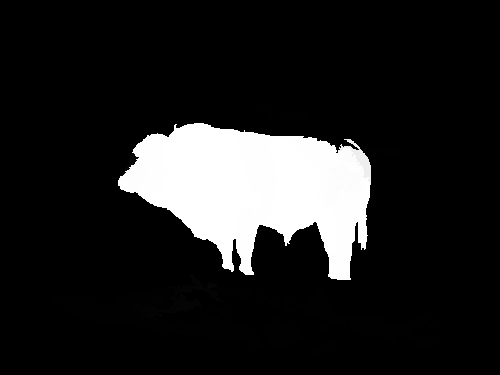}}       \\

\\ input & 100 superpixels & 300 superpixels & 800 superpixels & integrated \\

\end{tabular}
\end{center}
\caption{\textbf{Multi-scale saliency detection}: using \cite{qin2015saliency} with SH.
}
\label{fig:saliency}
\end{figure}

\section{Applications}
\label{sec:application}


To demonstrate the properties of proposed \emph{Super Hierarchy} (SH) method, we show how it is used as a preliminary that impacts three important computer vision tasks: semantic segmentation, saliency detection and stereo matching.



{\flushleft \textbf{Semantic Segmentation.}} Semantic segmentation aims at assigning pre-defined class labels to every pixel in an image.
One of the most successful frameworks for this task models the problem as an energy minimization of a conditional random field (CRF) \cite{gould2008multi,gonfaus2010harmony}. By working directly on the superpixel level instead of the pixel level, the number of nodes in the CRF is significantly reduced (typically from $10^5$ to $10^2$ per image \cite{gonfaus2010harmony}). Therefore, the inference algorithm converges drastically faster \cite{gonfaus2010harmony}.

Following \cite{achanta2012slic}, we use the method of \cite{gould2008multi} to evaluate superpixel algorithms on the MSRC-21 database~\cite{shotton2006textonboost}. The original annotations of MSRC-21 are quite imprecise and in order to get reliable results, we use an accurate version provided by \cite{malisiewicz-bmvc07}. All settings of \cite{gould2008multi} are kept constant for all superpixel methods. The results appearing in Table \ref{tab:semantic} show that SEEDS, LSC and proposed SH perform well on per-class evaluation while SH improves the per-class accuracy significantly. Some examples are shown in Figure \ref{fig:semantic_comparison}.

\begin{table}[t]
\begin{center}
\def\arraystretch{1.0}
\renewcommand{\tabcolsep}{5.0 pt}
\scalebox{0.8}{%
\begin{tabular}{lccc|cc|cc}
\cmidrule[\heavyrulewidth]{2-8}
\multirow{3}{1.5cm}{} & \multicolumn{7}{c}{Mean Absolute Error} \\
\cmidrule{2-8}
& \multicolumn{3}{c|}{PASCAL-S\cite{li2014secrets}} & \multicolumn{2}{c|}{ECSSD\cite{yan2013hierarchical}} & \multicolumn{2}{c}{DUT-OMRON\cite{yang2013saliency}} \\
& Single & Multi & Time & Single & Multi & Single & Multi \\
\midrule
FH~\cite{felzenszwalb2004efficient} & 0.227 & 0.187 & 1810 ms & 0.186 & 0.141 & 0.177 & 0.154 \\
SLIC~\cite{achanta2012slic}         & 0.225 & 0.190 & \,\,\,554 ms & 0.183 & 0.139 & 0.191 & 0.169 \\
ERS~\cite{liu2011entropy}           & 0.224 & 0.189 & 4160 ms & 0.182 & 0.138 & 0.189 & 0.166 \\
SEEDS~\cite{van2012seeds}           & 0.223 & 0.186 & \,\,\,282 ms & 0.179 & 0.137 & 0.178 & 0.155 \\
LSC~\cite{LiC15}                    & 0.226 & 0.190 & 1720 ms & 0.185 & 0.139 & 0.195 & 0.167 \\
SH                                  & \textbf{0.217} & \textbf{0.181} & \textbf{\,\,\,\,\,\,35 ms} & \textbf{0.178} & \textbf{0.133} & \textbf{0.172} & \textbf{0.146} \\

\bottomrule
\end{tabular}}
\end{center}

\caption{\textbf{Multi-scale segmentation for saliency detection}: using \cite{qin2015saliency} with different superpixel algorithms.
The multi-scale method integrates 5 scales of superpixels. Note that multi-scale superpixels clearly and consistently improve saliency detection for all methods.
SH performs best both in single-scale and multi-scale cases. Additionally, SH is significantly faster than others since all scales of superpixels are generated at one time. It is 8$\times$ faster than the second best (SEEDS) on PASCAL-S. }
\label{tab:saliency}
\end{table}

\begin{table*}[!htb]
\def\arraystretch{1.3}
\renewcommand{\tabcolsep}{2.0 pt}
\begin{center}
\scalebox{0.9}{%
\begin{tabular}{l|ccccccccccccccccccc|cc}
\midrule[\heavyrulewidth]
 & \rotatebox{60}{Tsukuba} & \rotatebox{60}{Venus} & \rotatebox{60}{Teddy} & \rotatebox{60}{Cones} & \rotatebox{60}{Aloe} & \rotatebox{60}{Art} & \rotatebox{60}{Baby1} & \rotatebox{60}{Baby2} & \rotatebox{60}{Baby3} & \rotatebox{60}{Books} & \rotatebox{60}{Cloth2} & \rotatebox{60}{Cloth3} & \rotatebox{60}{Dolls} & \rotatebox{60}{Flower.} & \rotatebox{60}{Lamp.} & \rotatebox{60}{Laundry} & \rotatebox{60}{Midd1} & \rotatebox{60}{Moebius} & \rotatebox{60}{Wood1} & \rotatebox{60}{Avg.Error} & \rotatebox{60}{Avg.Rank} \\
 \hline
 MST\cite{yang2012non}   & $\textbf{1.71}_1$ & $\textbf{0.64}_1$ & $7.14_2$ & $3.89_3$ & $4.46_3$ & $10.54_4$ & $8.89_4$ & $13.53_4$ & $6.37_4$ & $10.10_4$ & $3.61_3$ & $1.95_4$ & $5.70_3$ & $19.21_4$ & $11.41_4$ & $12.92_2$ & $30.99_2$ & $\textbf{7.92}_1$ & $10.13_4$ & $9.01_4$ & $3.00_3$ \\
 FH\cite{felzenszwalb2004efficient}    & $1.89_2$ & $0.76_2$ & $7.55_3$ & $3.64_2$ & $4.15_2$ & $10.51_3$ & $7.37_3$ & $11.28_3$ & $5.36_3$ & $09.05_2$ & $3.15_2$ & $1.58_2$ & $5.39_2$ & $15.73_3$ & $11.14_2$ & $\textbf{12.70}_1$ & $\textbf{24.92}_1$ & $8.16_4$ & $09.51_3$  & $8.10_2$ & $2.37_2$ \\
 ERS\cite{liu2011entropy}   & $2.65_4$ & $1.45_4$ & $8.87_4$ & $3.94_4$ & $4.57_4$ & $10.00_2$ & $5.78_2$ & $09.14_2$ & $5.02_2$ & $09.77_3$ & $4.03_4$ & $1.82_3$ & $5.86_4$ & $14.81_2$ & $\textbf{10.78}_1$ & $14.97_4$ & $38.89_4$ & $8.08_2$ & $04.72_2$ & $8.69_3$ & $3.00_3$ \\
 SH                        & $2.22_3$ & $1.15_3$ & $\textbf{7.05}_1$ & $\textbf{3.50}_1$ & $\textbf{3.27}_1$ & $\textbf{08.12}_1$ & $\textbf{4.93}_1$ & $\textbf{04.80}_1$ & $\textbf{4.69}_1$ & $\textbf{08.77}_1$ & $\textbf{2.18}_1$ & $\textbf{1.17}_1$ & $\textbf{4.77}_1$ & $\textbf{12.61}_1$ & $11.20_3$ & $14.63_3$ & $38.72_3$ & $8.08_2$ & $\textbf{03.12}_1$ & $\textbf{7.63}_1$ & $\textbf{1.58}_1$ \\
\bottomrule
\end{tabular}}
\end{center}

\caption{\textbf{Stereo matching evaluation on 19 Middlebury datasets \cite{middlebury}}. Results for the method of \cite{yang2012non} with 4 tree structures: MST \cite{yang2012non}, FH \cite{felzenszwalb2004efficient}, ERS \cite{liu2011entropy} and proposed SH.
Percentages of the erroneous pixels in non-occlusion regions with threshold 1 are used to evaluate the aggregation accuracy of the structures.
The subscripts represent the relative rank of the methods on each data set.
SH produces the most accurate disparity map on 13 data sets.}
\label{tab:stereo}
\end{table*}

\begin{figure}\scriptsize
\begin{center}

%
%

\def\arraystretch{0.6}
\renewcommand{\tabcolsep}{0.5 pt}
\setlength{\fboxrule}{0.5pt}
\setlength{\fboxsep}{0pt}
\def\swidth{0.24\linewidth}
\begin{tabular}{cccc}

\fbox{\includegraphics[width=\swidth]{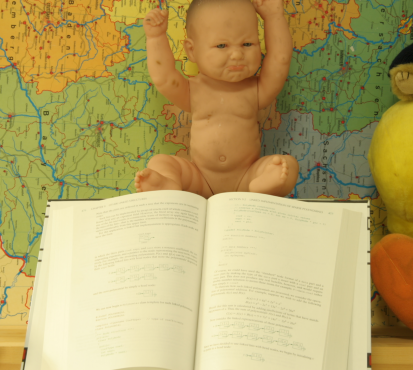}}    &
\fbox{\begin{overpic}[width=\swidth,unit=1mm]{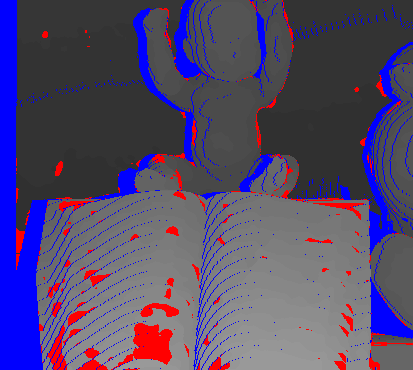}   \put(75,80){\color{white}{\textbf{ 4.80}}} \end{overpic}}  &
\fbox{\begin{overpic}[width=\swidth,unit=1mm]{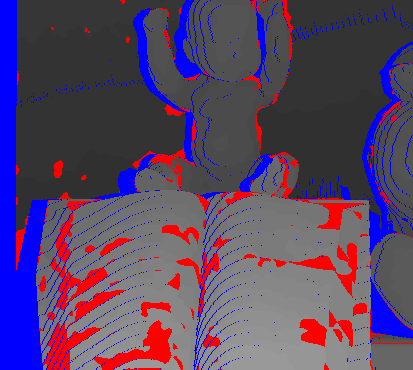}    \put(75,80){\color{white}{\textbf{ 9.14}}} \end{overpic}}  &
\fbox{\begin{overpic}[width=\swidth,unit=1mm]{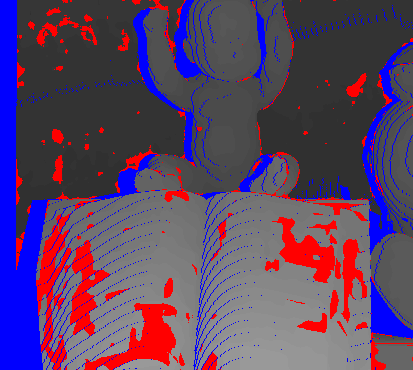}     \put(74,80){\color{white}{\textbf{11.28}}} \end{overpic}}  \\

\fbox{\includegraphics[width=\swidth]{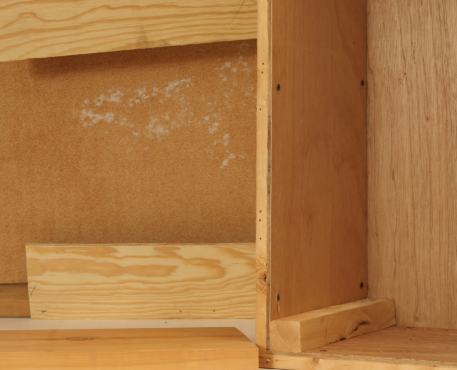}}    &
\fbox{\begin{overpic}[width=\swidth,unit=1mm]{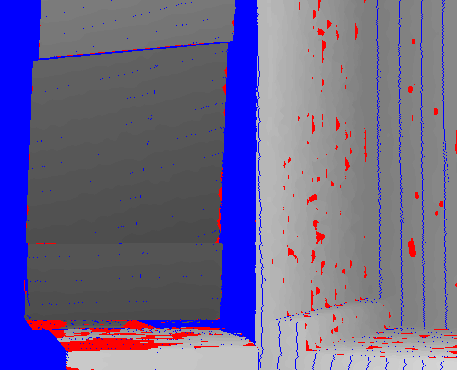} \put(75,72){\color{white}{\textbf{ 3.12}}} \end{overpic}}   &
\fbox{\begin{overpic}[width=\swidth,unit=1mm]{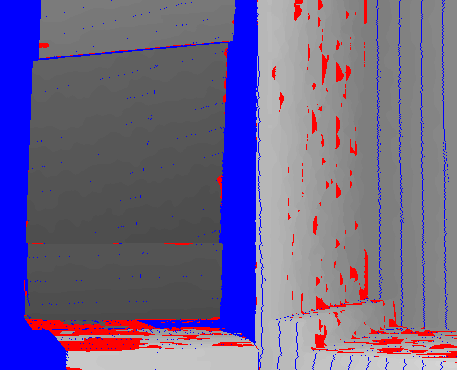}  \put(75,72){\color{white}{\textbf{ 4.72}}} \end{overpic}}   &
\fbox{\begin{overpic}[width=\swidth,unit=1mm]{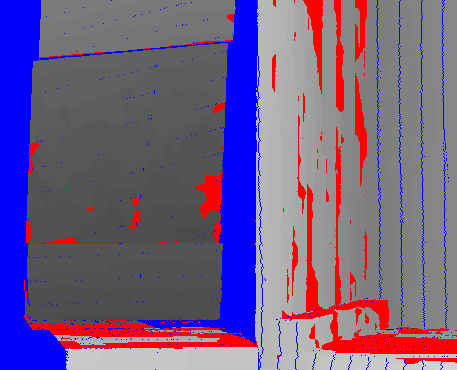}   \put(75,72){\color{white}{\textbf{ 9.51}}} \end{overpic}}   \\

\\ (a) left image & (b) SH & (c) ERS \cite{liu2011entropy} & (e) FH \cite{felzenszwalb2004efficient,mei2013segment} \\

\end{tabular}
\end{center}
\caption{\textbf{Stereo matching using different tree structures.} Results for the method of \cite{yang2012non} with different tree structures on Middlebury. Occlusion regions are marked in blue and erroneous pixels are marked in red. Numbers in the upper-right corner indicate percentages of bad pixels.}
\end{figure}


{\flushleft \textbf{Saliency Detection.}}
%
The goal of saliency detection is to tell whether a pixel belongs to the most salient object.
The method of \cite{qin2015saliency} introduces \emph{Cellular Automata} (CA) to intuitively detect the salient object. CA can be designed in a single-layer (SCA) or multi-layer (MCA) fashion.
It's shown in \cite{qin2015saliency} that MCA improves saliency detection accuracy significantly compared to SCA by fusing multiple saliency detection methods.
Here we demonstrate that improvement can also be achieved by fusing multi-scale segmentation.

We use 5 scales of superpixels range from 100 to 1000 for MCA. Results are compared on three challenging datasets: PASCAL-S \cite{li2014secrets}, ECSSD \cite{yan2013hierarchical} and DTU-OMERON \cite{yang2013saliency}. As shown in Table \ref{tab:saliency} and Figure \ref{fig:saliency}, multi-scale segmentations clearly and consistently improve saliency detection accuracy for all superpixel algorithms. SH shows striking advantages for this task as generating the most accurate saliency maps and reducing computational cost significantly.


{\flushleft \textbf{Stereo Matching.}}
To demonstrate the usefulness of tree structure provided by SH, we integrate it with the non-local cost aggregation method \cite{yang2012non} for stereo matching.
Different from traditional local stereo methods, \cite{yang2012non} performs cost aggregation over the entire image with a MST in a non-local manner.
The method is computationally very efficient, with a complexity comparable to uniform box filtering and also shows edge-preserving and non-local properties.

Following \cite{mei2013segment}, we quantitatively evaluate the aggregation accuracy with MST, FH, ERS and our SH on 19 Middlebury data sets.
All the methods use the same cost volume and do not employ any post-processing.
The disparity error rates in non-occlusion regions are used to evaluate. Table \ref{tab:stereo} shows the experimental results. The subscripts represent relative rank of the methods on each data set. As expected, all segmentation-based structures improve the basic MST. The performance of proposed SH is higher than the other tree structures. It obtains the lowest average error rate and the highest average ranking. SH achieves the most accurate results on 13 (out of 19) datasets.

\section{Conclusion}

This paper presents a simple yet effective hierarchical superpixel segmentation method that can be used in a wide range of computer vision tasks.
A great advantage of the proposed method is that it generates all scale of superpixels efficiently and accurately.
Future work includes speeding up the proposed method on GPU and applying it to point cloud and video sequence segmentation.

\ifCLASSOPTIONcaptionsoff
  \newpage
\fi

\bibliographystyle{IEEETran}
\bibliography{super}

\begin{thebibliography}{10}
\providecommand{\url}[1]{#1}
\csname url@samestyle\endcsname
\providecommand{\newblock}{\relax}
\providecommand{\bibinfo}[2]{#2}
\providecommand{\BIBentrySTDinterwordspacing}{\spaceskip=0pt\relax}
\providecommand{\BIBentryALTinterwordstretchfactor}{4}
\providecommand{\BIBentryALTinterwordspacing}{\spaceskip=\fontdimen2\font plus
\BIBentryALTinterwordstretchfactor\fontdimen3\font minus
  \fontdimen4\font\relax}
\providecommand{\BIBforeignlanguage}[2]{{%
\expandafter\ifx\csname l@#1\endcsname\relax
\typeout{** WARNING: IEEEtran.bst: No hyphenation pattern has been}%
\typeout{** loaded for the language `#1'. Using the pattern for}%
\typeout{** the default language instead.}%
\else
\language=\csname l@#1\endcsname
\fi
#2}}
\providecommand{\BIBdecl}{\relax}
\BIBdecl

\bibitem{DollarICCV13edges}
P.~Doll\'ar and C.~L. Zitnick, ``Structured forests for fast edge detection,''
  in \emph{ICCV}, 2013.

\bibitem{felzenszwalb2004efficient}
P.~F. Felzenszwalb and D.~P. Huttenlocher, ``Efficient graph-based image
  segmentation,'' \emph{IJCV}, 2004.

\bibitem{achanta2012slic}
R.~Achanta, A.~Shaji, K.~Smith, A.~Lucchi, P.~Fua, and S.~Susstrunk, ``Slic
  superpixels compared to state-of-the-art superpixel methods,'' \emph{TPAMI},
  2012.

\bibitem{liu2011entropy}
M.-Y. Liu, O.~Tuzel, S.~Ramalingam, and R.~Chellappa, ``Entropy rate superpixel
  segmentation,'' in \emph{CVPR}, 2011.

\bibitem{van2012seeds}
M.~Van~den Bergh, X.~Boix, G.~Roig, B.~de~Capitani, and L.~Van~Gool, ``Seeds:
  Superpixels extracted via energy-driven sampling,'' in \emph{ECCV}, 2012.

\bibitem{LiC15}
Z.~Li and J.~Chen, ``Superpixel segmentation using linear spectral
  clustering,'' in \emph{CVPR}, 2015.

\bibitem{pami-11-malik}
P.~Arbelaez, M.~Maire, C.~Fowlkes, and J.~Malik., ``Contour detection and
  hierarchical image segmentation,'' \emph{TPAMI}, 2011.

\bibitem{gould2008multi}
S.~Gould, J.~Rodgers, D.~Cohen, G.~Elidan, and D.~Koller, ``Multi-class
  segmentation with relative location prior,'' \emph{IJCV}, 2008.

\bibitem{shotton2006textonboost}
J.~Shotton, J.~Winn, C.~Rother, and A.~Criminisi, ``Textonboost: Joint
  appearance, shape and context modeling for multi-class object recognition and
  segmentation,'' in \emph{ECCV}, 2006.

\bibitem{qin2015saliency}
Y.~Qin, H.~Lu, Y.~Xu, and H.~Wang, ``Saliency detection via cellular
  automata,'' in \emph{CVPR}, 2015.

\bibitem{li2014secrets}
Y.~Li, X.~Hou, C.~Koch, J.~M. Rehg, and A.~L. Yuille, ``The secrets of salient
  object segmentation,'' in \emph{CVPR}, 2014.

\bibitem{yang2012non}
Q.~Yang, ``A non-local cost aggregation method for stereo matching,'' in
  \emph{CVPR}, 2012.

\bibitem{middlebury}
D.Scharstein and R.Szeliski, ``Middlebury stereo datasets,''
  \url{http://vision.middlebury.edu/stereo/data/}.

\bibitem{moore2008superpixel}
A.~P. Moore, S.~Prince, J.~Warrell, U.~Mohammed, and G.~Jones, ``Superpixel
  lattices,'' in \emph{CVPR}, 2008.

\bibitem{moore2010lattice}
A.~P. Moore, S.~J. Prince, and J.~Warrell, ``"lattice cut"-constructing
  superpixels using layer constraints,'' in \emph{CVPR}, 2010.

\bibitem{todorovic2008unsupervised}
S.~Todorovic and N.~Ahuja, ``Unsupervised category modeling, recognition, and
  segmentation in images,'' \emph{TPAMI}, 2008.

\bibitem{mei2013segment}
X.~Mei, X.~Sun, W.~Dong, H.~Wang, and X.~Zhang, ``Segment-tree based cost
  aggregation for stereo matching,'' in \emph{CVPR}, 2013.

\bibitem{neubert2012superpixel}
P.~Neubert and P.~Protzel, ``Superpixel benchmark and comparison,'' in
  \emph{Proc. Forum Bildverarbeitung}, 2012.

\bibitem{kohli2009robust}
P.~Kohli, P.~H. Torr \emph{et~al.}, ``Robust higher order potentials for
  enforcing label consistency,'' \emph{IJCV}, 2009.

\bibitem{yan2013hierarchical}
Q.~Yan, L.~Xu, J.~Shi, and J.~Jia, ``Hierarchical saliency detection,'' in
  \emph{CVPR}, 2013.

\bibitem{jiang2013salient}
H.~Jiang, J.~Wang, Z.~Yuan, Y.~Wu, N.~Zheng, and S.~Li, ``Salient object
  detection: A discriminative regional feature integration approach,'' in
  \emph{CVPR}, 2013.

\bibitem{akbas2010ramp}
E.~Akbas and N.~Ahuja, ``From ramp discontinuities to segmentation tree,'' in
  \emph{ACCV}, 2010.

\bibitem{ahuja1996transform}
N.~Ahuja, ``A transform for multiscale image segmentation by integrated edge
  and region detection,'' \emph{TPAMI}, 1996.

\bibitem{ahuja2008connected}
N.~Ahuja and S.~Todorovic, ``Connected segmentation tree—a joint representation
  of region layout and hierarchy,'' in \emph{CVPR}, 2008.

\bibitem{akbas2014low}
E.~Akbas and N.~Ahuja, ``Low-level hierarchical multiscale segmentation
  statistics of natural images,'' \emph{TPAMI}, 2014.

\bibitem{APBMM2014}
P.~Arbel\'{a}ez, J.~Pont-Tuset, J.~Barron, F.~Marques, and J.~Malik,
  ``Multiscale combinatorial grouping,'' in \emph{CVPR}, 2014.

\bibitem{gu2009recognition}
C.~Gu, J.~J. Lim, P.~Arbelaez, and J.~Malik, ``Recognition using regions,'' in
  \emph{CVPR}, 2009.

\bibitem{girshick2014rich}
R.~Girshick, J.~Donahue, T.~Darrell, and J.~Malik, ``Rich feature hierarchies
  for accurate object detection and semantic segmentation,'' in \emph{CVPR},
  2014.

\bibitem{salembier2000binary}
P.~Salembier and L.~Garrido, ``Binary partition tree as an efficient
  representation for image processing, segmentation, and information
  retrieval,'' \emph{TIP}, 2000.

\bibitem{calderero2010region}
F.~Calderero and F.~Marques, ``Region merging techniques using information
  theory statistical measures,'' \emph{TIP}, 2010.

\bibitem{west2001introduction}
D.~B. West \emph{et~al.}, \emph{Introduction to graph theory}.\hskip 1em plus
  0.5em minus 0.4em\relax Prentice Hall, 2001.

\bibitem{mares2004}
M.~Mareš, ``\BIBforeignlanguage{eng}{Two linear time algorithms for mst on
  minor closed graph classes},'' \emph{\BIBforeignlanguage{eng}{Archivum
  Mathematicum}}, 2004.

\bibitem{cormen2001introduction}
T.~H. Cormen, C.~E. Leiserson, R.~L. Rivest, C.~Stein \emph{et~al.},
  \emph{Introduction to algorithms}.\hskip 1em plus 0.5em minus 0.4em\relax MIT
  Press, 2001.

\bibitem{Everingham10}
M.~Everingham, L.~Van~Gool, C.~K.~I. Williams, J.~Winn, and A.~Zisserman, ``The
  pascal visual object classes (voc) challenge,'' \emph{IJCV}, 2010.

\bibitem{Hariharan2011}
B.~Hariharan, P.~Arbelaez, L.~Bourdev, S.~Maji, and J.~Malik, ``Semantic
  contours from inverse detectors,'' in \emph{ICCV}, 2011.

\bibitem{Lin2014}
T.-Y. Lin, M.~Maire, S.~Belongie, J.~Hays, P.~Perona, D.~Ramanan,
  P.~Doll{\'a}r, and C.~Zitnick, ``\BIBforeignlanguage{English}{{Microsoft
  COCO: Common Objects in Context}},'' in
  \emph{\BIBforeignlanguage{English}{ECCV}}, 2014.

\bibitem{xie15hed}
S.~Xie and Z.~Tu, ``Holistically-nested edge detection,'' in \emph{ICCV}, 2015.

\bibitem{ren_nips12}
X.~Ren and L.~Bo, ``{Discriminatively Trained Sparse Code Gradients for Contour
  Detection},'' in \emph{NIPS}, 2012.

\bibitem{gonfaus2010harmony}
J.~M. Gonfaus, X.~Boix, J.~Van~de Weijer, A.~D. Bagdanov, J.~Serrat, and
  J.~Gonzalez, ``Harmony potentials for joint classification and
  segmentation,'' in \emph{CVPR}, 2010.

\bibitem{malisiewicz-bmvc07}
T.~Malisiewicz and A.~A. Efros, ``Improving spatial support for objects via
  multiple segmentations,'' in \emph{BMVC}, 2007.

\bibitem{yang2013saliency}
C.~Yang, L.~Zhang, H.~Lu, X.~Ruan, and M.-H. Yang, ``Saliency detection via
  graph-based manifold ranking,'' in \emph{CVPR}, 2013.

\end{thebibliography}

\end{document}